\documentclass[english]{article}
\usepackage[numbers]{natbib}
\usepackage[final]{nips_2016}

\usepackage[utf8]{inputenc} 
\usepackage[T1]{fontenc}    
\usepackage{hyperref}       
\usepackage{url}            
\usepackage{booktabs}       
\usepackage{amsfonts}       
\usepackage{nicefrac}       
\usepackage{microtype}      

\usepackage{babel}
\usepackage{float}
\usepackage{amsmath}
\usepackage{amsthm}
\usepackage{amssymb}
\usepackage{graphicx}
\usepackage{esint}

\usepackage{algorithm,algpseudocode,dsfont}
\newtheorem{cond}{Condition}
\newtheorem{prop}{Proposition}

\floatstyle{ruled}
\newfloat{algorithm}{tbp}{loa}
\providecommand{\algorithmname}{Algorithm}
\floatname{algorithm}{\protect\algorithmname}

\theoremstyle{plain}
\newtheorem{thm}{\protect\theoremname}
  \theoremstyle{definition}
  \newtheorem{defn}{\protect\definitionname}
  \theoremstyle{definition*}
  \newtheorem*{defn*}{\protect\definitionname}
  \theoremstyle{definition*}
  \newtheorem*{notation*}{Notation}
  \theoremstyle{plain}
  \newtheorem{lem}{\protect\lemmaname}
  \theoremstyle{remark}
  \newtheorem*{note*}{\protect\notename}

\makeatother

  \providecommand{\definitionname}{Definition}
  \providecommand{\lemmaname}{Lemma}
  \providecommand{\notename}{Note}
\providecommand{\theoremname}{Theorem}

\title{Active Learning from Imperfect Labelers}

\author{
Songbai Yan \\
University of California, San Diego \\
\texttt{yansongbai@eng.ucsd.edu} \\
\And
Kamalika Chaudhuri \\
University of California, San Diego \\
\texttt{kamalika@cs.ucsd.edu} \\
\And
Tara Javidi \\
University of California, San Diego \\
\texttt{tjavidi@eng.ucsd.edu} \\
}

\begin{document}

\maketitle
\begin{abstract}
We study active learning where the labeler can not only return incorrect
labels but also abstain from labeling. We consider different noise
and abstention conditions of the labeler. We propose an algorithm
which utilizes abstention responses, and analyze its statistical consistency
and query complexity under fairly natural assumptions on the noise
and abstention rate of the labeler. This algorithm is adaptive in
a sense that it can automatically request less queries with a more
informed or less noisy labeler. We couple our algorithm with lower
bounds to show that under some technical conditions, it achieves nearly
optimal query complexity. 
\end{abstract}

\section{Introduction}

In active learning, the learner is given an input space $\mathcal{X}$,
a label space $\mathcal{L}$, and a hypothesis class $\mathcal{H}$
such that one of the hypotheses in the class generates ground truth
labels. Additionally, the learner has at its disposal a labeler to
which it can pose interactive queries about the labels of examples
in the input space. Note that the labeler may output a noisy version
of the ground truth label (a flipped label). The goal of the learner
is to learn a hypothesis in $\mathcal{H}$ which is close to the hypothesis
that generates the ground truth labels. 

There has been a significant amount of literature on active learning,
both theoretical and practical. Previous theoretical work on active
learning has mostly focused on the above basic setting~\cite{BBL06,BHLZ10,chen2015sequential,DHM07,ZC14}
and has developed algorithms under a number of different models of
label noise. A handful of exceptions include \cite{balcan2012robust}~which
allows class conditional queries, \cite{beygelzimer2016search}~which
allows requesting counterexamples to current version spaces, and \cite{urner2012learning,zhang2015active}~
where the learner has access to a strong labeler and one or more weak
labelers.

In this paper, we consider a more general setting where, in addition
to providing a possibly noisy label, the labeler can sometimes abstain
from labeling. This scenario arises naturally in difficult labeling
tasks and has been considered in computer vision by~\cite{fang2012don,kading2015active}.
Our goal in this paper is to investigate this problem from a foundational
perspective, and explore what kind of conditions are needed, and how
an abstaining labeler can affect properties such as consistency and
query complexity of active learning algorithms. 

The setting of active learning with an abstaining noisy labeler was
first considered by~\cite{YCJ2015}, who looked at learning binary
threshold classifiers based on queries to an labeler whose abstention
rate is higher closer to the decision boundary. They primarily looked
at the case when the abstention rate at a distance $\Delta$ from
the decision boundary is less than $1-\Theta(\Delta^{\alpha})$, and
the rate of label flips at the same distance is less than $\frac{1}{2}-\Theta(\Delta^{\beta})$;
under these conditions, they provided an active learning algorithm
that given parameters $\alpha$ and $\beta$, outputs a classifier
with error $\epsilon$ using $\tilde{O}(\epsilon^{-\alpha-2\beta})$
queries to the labeler. However, there are several limitations to
this work. The primary limitation is that parameters $\alpha$ and
$\beta$ need to be known to the algorithm, which is not usually the
case in practice. A second major limitation is that even if the labeler
has nice properties, such as, the abstention rates increase sharply
close to the boundary, their algorithm is unable to exploit these
properties to reduce the number of queries. A third and final limitation
is that their analysis only applies to one dimensional thresholds,
and not to more general decision boundaries. 

In this work, we provide an algorithm which is able to exploit nice
properties of the labeler. Our algorithm is statistically consistent
under very mild conditions — when the abstention rate is non-decreasing
as we get closer to the decision boundary. Under slightly stronger
conditions as in~\cite{YCJ2015}, our algorithm has the same query
complexity. However, if the abstention rate of the labeler increases
strictly monotonically close to the decision boundary, then our algorithm
adapts and does substantially better. It simply exploits the increasing
abstention rate close to the decision boundary, and does not even
have to rely on the noisy labels! Specifically, when applied to the
case where the noise rate is at most $\frac{1}{2}-\Theta(\Delta^{\beta})$
and the abstention rate is $1-\Theta(\Delta^{\alpha})$ at distance
$\Delta$ from the decision boundary, our algorithm can output a classifier
with error $\epsilon$ based on only $\tilde{O}(\epsilon^{-\alpha})$
queries. 

An important property of our algorithm is that the improvement of
query complexity is achieved in a \emph{completely adaptive manner};
unlike previous work~\cite{YCJ2015}, our algorithm needs \emph{no
information whatsoever on the abstention rates or rates of label noise}.
Thus our result also strengthens existing results on active learning
from (non-abstaining) noisy labelers by providing an adaptive algorithm
that achieves that same performance as~\cite{CN08} without knowledge
of noise parameters. 

We extend our algorithm so that it applies to any smooth $d$-dimensional
decision boundary in a non-parametric setting, not just one-dimensional
thresholds, and we complement it with lower bounds on the number of
queries that need to be made to any labeler. Our lower bounds generalize
the lower bounds in~\cite{YCJ2015}, and shows that our upper bounds
are nearly optimal. We also present an example that shows that at
least a relaxed version of the monotonicity property is necessary
to achieve this performance gain; if the abstention rate plateaus
around the decision boundary, then our algorithm needs to query and
rely on the noisy labels (resulting in higher query complexity) in
order to find a hypothesis close to the one generating the ground
truth labels. 

\subsection{Related work}

There has been a considerable amount of work on active learning, most
of which involves labelers that are not allowed to abstain. Theoretical
work on this topic largely falls under two categories — the membership
query model~\cite{CN08,hegedHus1995generalized,NJC15,N11}, where
the learner can request label of any example in the instance space,
and the PAC model, where the learner is given a large set of unlabeled
examples from an underlying unlabeled data distribution, and can request
labels of a subset of these examples. Our work and also that of~\cite{YCJ2015}
builds on the membership query model.

There has also been a lot of work on active learning under different
noise models. The problem is relatively easy when the labeler always
provides the ground truth labels – see~\cite{CAL94,D05,hanneke2007teaching}
for work in this setting in the PAC model, and~\cite{hegedHus1995generalized}
for the membership query model. Perhaps the simplest setting of label
noise is random classification noise, where each label is flipped
with a probability that is independent of the unlabeled instance.
\cite{K06} shows how to address this kind of noise in the PAC model
by repeatedly querying an example until the learner is confident of
its label;~\cite{NJC15,N11} provide more sophisticated algorithms
with better query complexities in the membership query model. A second
setting is when the noise rate increases closer to the decision boundary;
this setting has been studied under the membership query model by~\cite{CN08}
and in the PAC model by~\cite{DHM07,BHLZ10,ZC14}. A final setting
is agnostic PAC learning — when a fixed but arbitrary fraction of
labels may disagree with the label assigned by the optimal hypothesis
in the hypothesis class. Active learning is known to be particularly
difficult in this setting; however, algorithms and associated label
complexity bounds have been provided by~\cite{BL13,BBL06,BHLZ10,DHM07,hanneke2007teaching,ZC14}
among others.

Our work expands on the membership query model, and our abstention
and noise models are related to a variant of the Tsybakov noise condition.
A setting similar to ours was considered by~\cite{CN08,YCJ2015}.
\cite{CN08} considers a non-abstaining labeler, and provides a near-optimal
binary search style active learning algorithm; however, their algorithm
is non-adaptive. \cite{YCJ2015} gives a nearly matching lower and
upper query complexity bounds for active learning with abstention
feedback, but they only give a non-adaptive algorithm for learning
one dimensional thresholds, and only study the situation where the
abstention rate is upper-bounded by a polynomial function. Besides
\cite{YCJ2015} , \cite{fang2012don,kading2015active} study active
learning with abstention feedback in computer vision applications.
However, these works are based on heuristics and do not provide any
theoretical guarantees.

\section{Settings\label{sec:Settings}}
\begin{notation*}
$\mathds{1}\left[A\right]$ is the indicator function: $\mathds{1}\left[A\right]=1$
if $A$ is true, and 0 otherwise. For $\boldsymbol{x}=(x_{1},\dots,x_{d})\in\mathbb{R}^{d}$
($d>1$), denote $(x_{1},\dots,x_{d-1})$ by $\boldsymbol{\tilde{x}}$.
Define $\ln x=\log_{e}x$, $\log x=\log_{\frac{4}{3}}x$, $\left[\ln\ln\right]_{+}(x)=\ln\ln\max\{x,e^{e}\}$.
We use $\tilde{O}$ and $\tilde{\Theta}$ to hide logarithmic factors
in $\frac{1}{\epsilon}$, $\frac{1}{\delta}$, and $d$.
\end{notation*}
\begin{defn*}
Suppose $\gamma\geq1$. A function $g:[0,1]^{d-1}\rightarrow\mathbb{R}$
is \emph{$(K,\gamma)$-Hölder smooth}, if it is continuously differentiable
up to $\left\lfloor \gamma\right\rfloor $-th order, and for any $\boldsymbol{x},\boldsymbol{y}\in[0,1]^{d-1}$,
$\left|g(\boldsymbol{y})-\sum_{m=0}^{\left\lfloor \gamma\right\rfloor }\frac{\partial^{m}g(\boldsymbol{x})}{m!}(\boldsymbol{y}-\boldsymbol{x})^{m}\right|\leq K\left\Vert \boldsymbol{y}-\boldsymbol{x}\right\Vert ^{\gamma}$.
We denote this class of functions by $\Sigma(K,\gamma)$.
\end{defn*}
We consider active learning for binary classification. We are given
an instance space $\mathcal{X}=[0,1]^{d}$ and a label space $\mathcal{L}=\{0,1\}$.
Each instance $x\in\mathcal{X}$ is assigned to a label $l\in\left\{ 0,1\right\} $
by an underlying function $h^{*}:\mathcal{X}\rightarrow\left\{ 0,1\right\} $
unknown to the learning algorithm in a hypothesis space $\mathcal{H}$
of interest. The learning algorithm has access to any $x\in\mathcal{X}$,
but no access to their labels. Instead, it can only obtain label information
through interactions with a labeler, whose relation to $h^{*}$ is
to be specified later. The objective of the algorithm is to sequentially
select the instances to query for label information and output a classifier
$\hat{h}$ that is close to $h^{*}$ while making as few queries as
possible.

We consider a non-parametric setting as  in \cite{CN08,minsker2012plug}
where the hypothesis space is the \emph{smooth boundary fragment}
class $\mathcal{H}=\{h_{g}(\boldsymbol{x})=\mathds{1}\left[x_{d}>g(\boldsymbol{\tilde{x}})\right]\mid g:[0,1]^{d-1}\rightarrow[0,1]\text{ is }(K,\gamma)\text{-Hölder smooth}\}$.
In other words, the decision boundaries of classifiers in this class
are epigraph of smooth functions (see Figure \ref{fig:induced-classifier}
for example). We assume $h^{*}(\boldsymbol{x})=\mathds{1}\left[x_{d}>g^{*}(\boldsymbol{\tilde{x}})\right]\in\mathcal{H}$.
When $d=1$, $\mathcal{H}$ reduces to the space of threshold functions
$\{h_{\theta}(x)=\mathds{1}\left[x>\theta\right]:\theta\in[0,1]\}$. 

The performance of a classifier $h(\boldsymbol{x})=\mathds{1}\left[x_{d}>g(\boldsymbol{\tilde{x}})\right]$
is evaluated by the $L^{1}$ distance between the decision boundaries
$\left\Vert g-g^{*}\right\Vert =\int_{[0,1]^{d-1}}\left|g(\boldsymbol{\tilde{x}})-g^{*}(\boldsymbol{\tilde{x}})\right|d\boldsymbol{\tilde{x}}$.

The learning algorithm can only obtain label information by querying
a labeler who is allowed to abstain from labeling or return an incorrect
label (flipping between 0 and 1). For each query $\boldsymbol{x}\in[0,1]^{d}$,
the labeler $L$ will return $y\in\mathcal{Y}=\{0,1,\perp\}$ ($\perp$
means that the labeler abstains from providing a 0/1 label) according
to some distribution $P_{L}(Y=y\mid X=\boldsymbol{x})$. When it is
clear from the context, we will drop the subscript from $P_{L}(Y\mid X)$.
Note that while the labeler can declare its indecision by outputting
$\perp$, we do not allow classifiers in our hypothesis space to output
$\perp$. 

In our active learning setting, our goal is to output a boundary $g$
that is close to $g^{*}$ while making as few interactive queries
to the labeler as possible. In particular, we want to find an algorithm
with low \emph{query complexity} $\Lambda(\epsilon,\delta,\mathcal{A},L,g^{*})$,
which is defined as the minimum number of queries that Algorithm $\mathcal{A}$,
acting on samples with ground truth $g^{*}$, should make to a labeler
$L$ to ensure that the output classifier $h_{g}(\boldsymbol{x})=\mathds{1}\left[x_{d}>g(\boldsymbol{\tilde{x}})\right]$
has the property $\left\Vert g-g^{*}\right\Vert =\int_{[0,1]^{d-1}}\left|g(\boldsymbol{\tilde{x}})-g^{*}(\boldsymbol{\tilde{x}})\right|d\boldsymbol{\tilde{x}}\leq\epsilon$
with probability at least $1-\delta$ over the responses of $L$.

\subsection{\label{subsec:Assumptions}Conditions}

We now introduce three conditions on the response of the labeler with
increasing strictness. Later we will provide an algorithm whose query
complexity improves with increasing strictness of conditions.

\begin{cond}\label{cond:oracle-correct}

The response distribution of the labeler $P(Y\mid X)$ satisfies: 

\begin{itemize} 

\item (abstention) For any $\boldsymbol{\tilde{x}}\in[0,1]^{d-1}$,
$x_{d},x_{d}'\in[0,1]$, if $\left|x_{d}-g^{*}(\boldsymbol{\tilde{x}})\right|\geq\left|x_{d}'-g^{*}(\boldsymbol{\tilde{x}})\right|$
then $P(\perp\mid(\boldsymbol{\tilde{x}},x_{d}))\leq P(\perp\mid(\boldsymbol{\tilde{x}},x_{d}'))$;

\item (noise) For any $\boldsymbol{x}\in[0,1]^{d}$, $P(Y\neq\mathds{1}\left[x_{d}>g^{*}(\boldsymbol{\tilde{x}})\right]\mid\boldsymbol{x},Y\neq\perp)\leq\frac{1}{2}$. 

\end{itemize}\end{cond}

Condition \ref{cond:oracle-correct} means that the closer $\boldsymbol{x}$
is to the decision boundary $\left(\boldsymbol{\tilde{x}},g^{*}(\boldsymbol{\tilde{x}})\right)$,
the more likely the labeler is to abstain from labeling. This complies
with the intuition that instances closer to the decision boundary
are harder to classify. We also assume the 0/1 labels can be flipped
with probability as large as $\frac{1}{2}$. In other words, we allow
unbounded noise.

\begin{cond}\label{cond:oracle-noshape}

Let $C,\beta$ be non-negative constants, and $f:[0,1]\rightarrow[0,1]$
be a nondecreasing function. The response distribution $P(Y\mid X)$
satisfies: 

\begin{itemize} 

\item (abstention) $P(\perp\mid\boldsymbol{x})\leq1-f\left(\left|x_{d}-g^{*}(\boldsymbol{\tilde{x}})\right|\right)$;

\item (noise) $P(Y\neq\mathds{1}\left[x_{d}>g^{*}(\boldsymbol{\tilde{x}})\right]\mid\boldsymbol{x},Y\neq\perp)\leq\frac{1}{2}\left(1-C\left|x_{d}-g^{*}(\boldsymbol{\tilde{x}})\right|^{\beta}\right)$. 

\end{itemize}

\end{cond}

Condition \ref{cond:oracle-noshape} requires the abstention and noise
probabilities to be upper-bounded, and these upper bounds decrease
as $\boldsymbol{x}$ moves further away from the decision boundary.
The abstention rate can be 1 at the decision boundary, so the labeler
may always abstain at the decision boundary. The condition on the
noise satisfies the popular Tsybakov noise condition \cite{T04}.

\begin{cond}\label{cond:oracle-shape}

Let $f:[0,1]\rightarrow[0,1]$ be a nondecreasing function such that
$\exists0<c<1$, $\forall0<a\leq1$ $\forall0\leq b\leq\frac{2}{3}a$,
$\frac{f(b)}{f(a)}\leq1-c$. The response distribution satisfies:
$P(\perp\mid\boldsymbol{x})=1-f\left(\left|x_{d}-g^{*}(\boldsymbol{\tilde{x}})\right|\right)$.

\end{cond}

An example where Condition \ref{cond:oracle-shape} holds is $P(\perp\mid\boldsymbol{x})=1-\left(x-0.3\right)^{\alpha}$
($\alpha>0$).

Condition~\ref{cond:oracle-shape} requires the abstention rate to
increase monotonically close to the decision boundary as in Condition
\ref{cond:oracle-correct}. In addition, it requires the abstention
probability $P(\perp|(\boldsymbol{\tilde{x}},x_{d}))$ not to be too
flat with respect to $x_{d}$. For example, when $d=1$, $P(\perp\mid x)=0.68$
for $0.2\leq x\leq0.4$ (shown as Figure \ref{fig:bad}) does not
satisfy Condition \ref{cond:oracle-shape}, and abstention responses
are not informative since this abstention rate alone yields no information
on the location of the decision boundary. In contrast, $P(\perp\mid x)=1-\sqrt{\left|x-0.3\right|}$
(shown as Figure \ref{fig:good}) satisfies Condition \ref{cond:oracle-shape},
and the learner could infer it is getting close to the decision boundary
when it starts receiving more abstention responses.

Note that here $c,f,C,\beta$ are \emph{unknown} and \emph{arbitrary}
parameters that characterize the complexity of the learning task.
We want to design an algorithm that does not require knowledge of
these parameters but still achieves nearly optimal query complexity.


\begin{figure}
\begin{minipage}[t]{0.3\columnwidth}%
\begin{center}
\includegraphics[width=1\columnwidth]{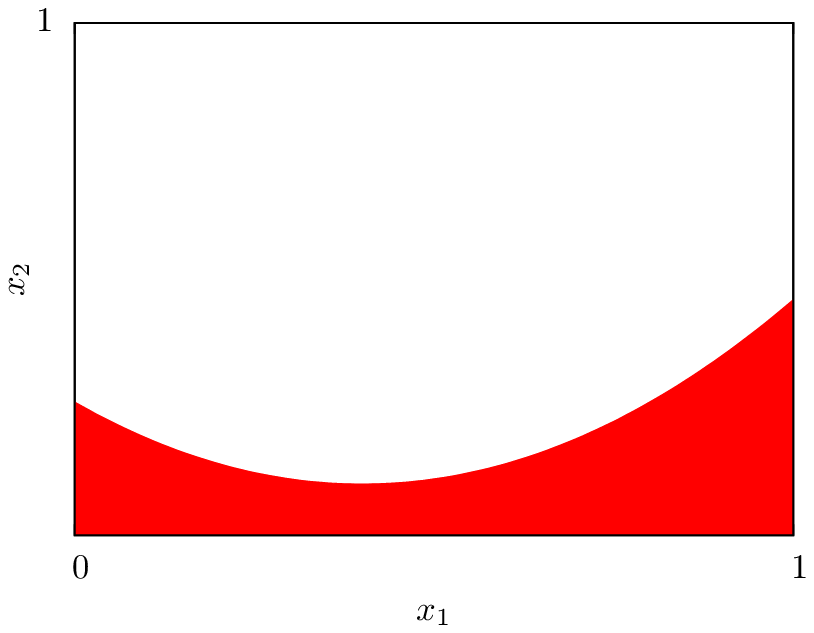}\caption{\label{fig:induced-classifier}A classifier with boundary $g(\boldsymbol{\tilde{x}})=\left(x_{1}-0.4\right)^{2}+0.1$
for $d=2$. Label 1 is assigned to the region above, 0 to the below
(red region)}
\par\end{center}%
\end{minipage}\hfill{}%
\begin{minipage}[t]{0.3\columnwidth}%
\begin{center}
\includegraphics[width=1\columnwidth]{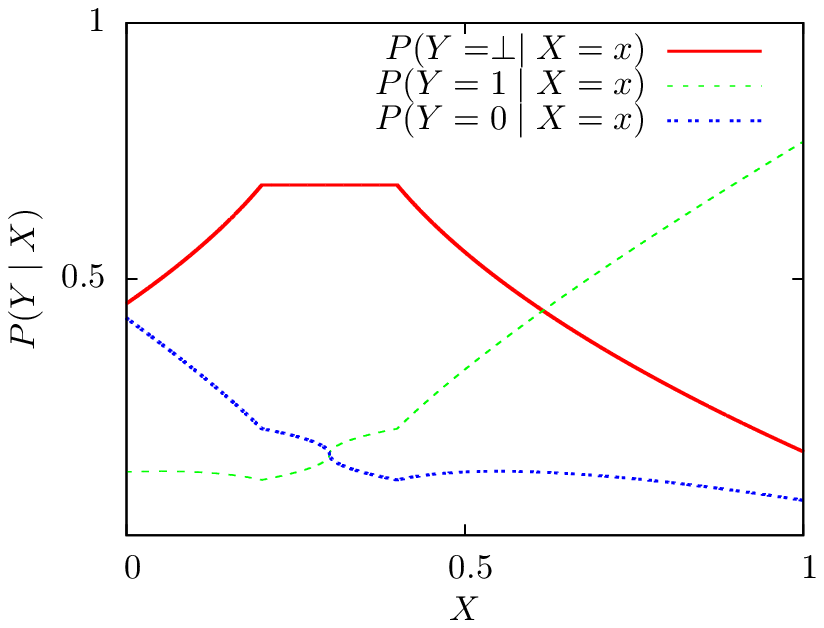}
\par\end{center}
\caption{\label{fig:bad}The distributions above satisfy Conditions~ \ref{cond:oracle-correct}
and \ref{cond:oracle-noshape}, but the abstention feedback is useless
since $P(\perp\mid x)$ is flat between $x=0.2$ and 0.4}
\end{minipage}\hfill{}%
\begin{minipage}[t]{0.3\columnwidth}%
\begin{center}
\includegraphics[width=1\columnwidth]{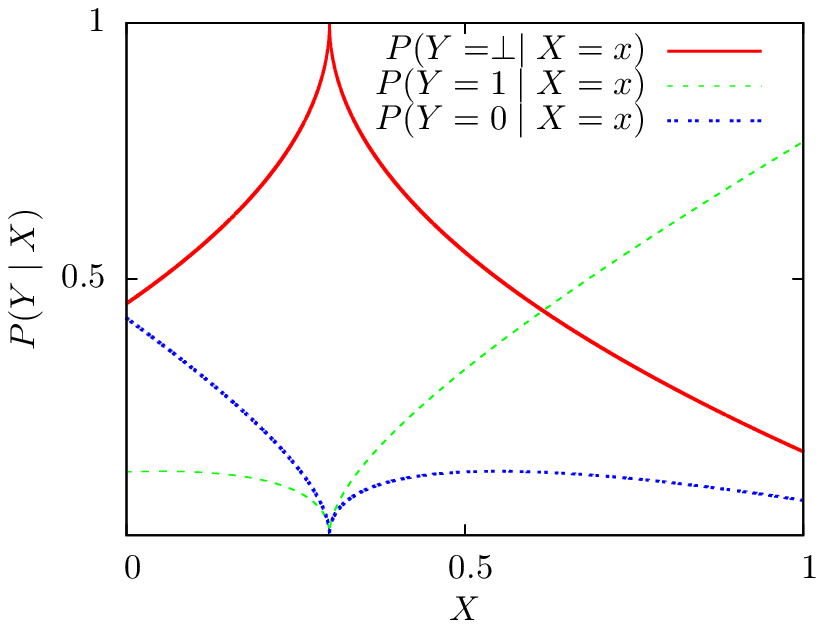}
\par\end{center}
\caption{\label{fig:good}Distributions above satisfy Conditions~ \ref{cond:oracle-correct},
\ref{cond:oracle-noshape}, and \ref{cond:oracle-shape}.}
\end{minipage}
\end{figure}

\section{Learning one-dimensional thresholds}

In this section, we start with the one dimensional case ($d=1$) to
demonstrate the main idea. We will generalize these results to multidimensional
instance space in the next section.

When $d=1$, the decision boundary $g^{*}$ becomes a point in $[0,1]$,
and the corresponding classifier is a threshold function over {[}0,1{]}.
In other words the hypothesis space becomes $\mathcal{H}=\{f_{\theta}(x)=\mathds{1}\left[x>\theta\right]:\theta\in[0,1]\}$).
We denote the ground truth decision boundary by $\theta^{*}\in[0,1]$.
We want to find a $\hat{\theta}\in[0,1]$ such that $|\hat{\theta}-\theta^{*}|$
is small while making as few queries as possible.

\subsection{Algorithm}

The proposed algorithm is a binary search style algorithm shown as
Algorithm~\ref{alg:The-active-learning}. (For the sake of simplicity,
we assume $\log\frac{1}{2\epsilon}$ is an integer.) Algorithm~\ref{alg:The-active-learning}
takes a desired precision $\epsilon$ and confidence level $\delta$
as its input, and returns an estimation $\hat{\theta}$ of the decision
boundary $\theta^{*}$. The algorithm maintains an interval $\left[L_{k},R_{k}\right]$
in which $\theta^{*}$ is believed to lie, and shrinks this interval
iteratively. To find the subinterval that contains $\theta^{*}$,
Algorithm~\ref{alg:The-active-learning} relies on two auxiliary
functions (marked in Procedure~\ref{alg:test}) to conduct adaptive
sequential hypothesis tests regarding subintervals of interval $\left[L_{k},R_{k}\right]$.

\begin{algorithm}
\begin{algorithmic}[1]
\State{Input: $\delta$, $\epsilon$}
\State{$[L_0,R_0]\gets [0,1]$} 
\For{$k=0,1,2,\dots,\log{\frac{1}{2\epsilon}}-1$} \label{line:algo-1d:iter} 
	\State{Define three quartiles: $U_k \gets \frac{3L_k+R_k}{4}$, $M_k \gets \frac{L_k+R_k}{2}$, $V_k \gets \frac{L_k+3R_k}{4}$} 
    \State{$A^{(u)},A^{(m)},A^{(v)},B^{(u)},B^{(v)} \gets $ Empty Array} 
	\For{$n=1,2,\dots $} \label{line:algo-1d:loop-n} 
		\State{Query at $U_k, M_k, V_k$, and receive labels $X^{(u)}_n, X^{(m)}_n, X^{(v)}_n$} 
		\For{$w\in \{u,m,v\}$} 
			\State \(\triangleright\) {We record whether $X^{(w)}=\perp$ in $A^{(w)}$, and the 0/1 label (as -1/1) in $B^{(w)}$ if $X^{(w)}\neq\perp$}
			\If{$X^{(w)}\neq\perp$} 	
				\State{$A^{(w)} \gets A^{(w)}\text{.append(1)}$} , $B^{(w)} \gets B^{(w)}\text{.append(}2\mathds{1}\left[X^{(w)}=1\right]-1\text{)}$
			\Else 	
				\State{$A^{(w)} \gets A^{(w)}\text{.append(0)}$} 
			\EndIf 
		\EndFor

		\State \(\triangleright\) {Check if the differences of abstention responses are statistically significant} \label{line:algo-1d:test-start}  
		\If{\Call{CheckSignificant-Var}{$\left\{ A^{(u)}_i - A^{(m)}_i\right\}_{i=1}^n$, $\frac{\delta}{4\log\frac{1}{2\epsilon}}$}} 
			\State{$[L_{k+1}, R_{k+1}] \gets [U_k, R_k]$; break} 
		\ElsIf{\Call{CheckSignificant-Var}{$\left\{ A^{(v)}_i - A^{(m)}_i\right\}_{i=1}^n$, $\frac{\delta}{4\log\frac{1}{2\epsilon}}$}} 
			\State{$[L_{k+1}, R_{k+1}] \gets [L_k, V_k]$; break} 
		\EndIf 

		\State \(\triangleright\) {Check if the differences between 0 and 1 labels are statistically significant} \label{line:algo-1d:test-label-start}
		\If{\Call{CheckSignificant}{$\left\{ -B^{(u)}_i\right\}_{i=1}^{B^{(u)}\text{.length}}$, $\frac{\delta}{4\log\frac{1}{2\epsilon}}$}} 
			\State{$[L_{k+1}, R_{k+1}] \gets [U_k, R_k]$; break}
		\ElsIf{\Call{CheckSignificant}{$\left\{ B^{(v)}_i\right\}_{i=1}^{B^{(v)}\text{.length}}$, $\frac{\delta}{4\log\frac{1}{2\epsilon}}$}}
			\State{$[L_{k+1}, R_{k+1}] \gets [L_k, V_k]$; break}
		\EndIf
\label{line:algo-1d:test-end}
	\EndFor
\EndFor
\State{Output: $\hat{\theta} = \left(L_{\log{\frac{1}{2\epsilon}}} + R_{\log{\frac{1}{2\epsilon}}} \right) /2$}
\end{algorithmic}

\caption{\label{alg:The-active-learning}The active learning algorithm for
learning thresholds}
\end{algorithm}

\floatname{algorithm}{Procedure}

\begin{algorithm}
\begin{algorithmic}[1] 
\State \(\triangleright\) {$D_0, D_1$ are absolute constants defined in  Proposition~\ref{prop:uniform-berstein} and Proposition~\ref{prop:uniform-empirical-berstein}}
\State \(\triangleright\) {$\left\{ X_i\right\}$ are i.i.d.\ random variables bounded by 1. $\delta$ is the confidence level. Detect if $\mathbb{E}X > 0$}
\Function{CheckSignificant}{$\left\{ X_i\right\}_{i=1}^n, \delta$}  
	\State{$p(n,\delta) \gets D_0 \left(1+ \ln\frac{1}{\delta} + \sqrt{4n \left( \left[\ln\ln\right]_{+}4n + \ln\frac{1}{\delta} \right)} \right)$}
	\State{Return $\sum_{i=1}^n X_i \geq p(n,\delta) $} 
\EndFunction
\Function{CheckSignificant-Var}{$\left\{ X_i\right\}_{i=1}^n, \delta$}  
	\State{Calculate the empirical variance  
		$\text{Var}=\frac{n}{n-1}\left( \sum_{i=1}^{n}{{X_i}^2}-\frac{1}{n}\left(\sum_{i=1}^{n}{X_i}\right)^2 \right)$} 
	\State{$q(n,\text{Var},\delta) \gets  D_{1}\left(1+\ln\frac{1}{\delta}+\sqrt{\left(\text{Var}+\ln\frac{1}{\delta}+1\right)\left(\left[\ln\ln\right]_{+}\left(\text{Var}+\ln\frac{1}{\delta}+1\right)+\ln\frac{1}{\delta}\right)}\right)$}
	\State{Return $n \geq \ln\frac{1}{\delta}$ AND $\sum_{i=1}^n X_i \geq q(n,\text{Var},\delta)$} 
\EndFunction 
\end{algorithmic}

\caption{\label{alg:test}Adaptive sequential testing}
\end{algorithm}

\floatname{algorithm}{Algorithm}

Suppose $\theta^{*}\in\left[L_{k},R_{k}\right]$. Algorithm~\ref{alg:The-active-learning}
tries to shrink this interval to a $\frac{3}{4}$ of its length in
each iteration by repetitively querying on quartiles $U_{k}=\frac{3L_{k}+R_{k}}{4}$,
$M_{k}=\frac{L_{k}+R_{k}}{2}$, $V_{k}=\frac{L_{k}+3R_{k}}{4}$. To
determine which specific subinterval to choose, the algorithm uses
0/1 labels and abstention responses simultaneously. Since the ground
truth labels are determined by $\mathds{1}\left[x>\theta^{*}\right]$,
one can infer that if the number of queries that return label 0 at
$U_{k}$ ($V_{k}$) is statistically significantly more (less) than
label 1, then $\theta^{*}$ should be on the right (left) side of
$U_{k}$ ($V_{k}$). Similarly, from Condition~\ref{cond:oracle-correct},
if the number of non-abstention responses at $U_{k}$ ($V_{k}$) is
statistically significantly more than non-abstention responses at
$M_{k}$, then $\theta^{*}$ should be closer to $M_{k}$ than $U_{k}$
($V_{k}$). 

Algorithm~\ref{alg:The-active-learning} relies on the ability to
shrink the search interval via statistically comparing the numbers
of obtained labels at locations $U_{k},M_{k},V_{k}$. As a result,
a main building block of Algorithm~\ref{alg:The-active-learning}
is to test whether i.i.d.\ bounded random variables $Y_{i}$ are
greater in expectation than i.i.d.\ bounded random variables $Z_{i}$
with statistical significance. In Procedure~\ref{alg:test}, we have
two test functions CheckSignificant and CheckSignificant-Var that
take i.i.d. random variables $\left\{ X_{i}=Y_{i}-Z_{i}\right\} $
($\left|X_{i}\right|\leq1$) and confidence level $\delta$ as their
input, and output whether it is statistically significant to conclude
$\mathbb{E}X_{i}>0$. 

CheckSignificant is based on the following uniform concentration result
regarding the empirical mean:
\begin{prop}
\label{prop:uniform-berstein}Suppose $X_{1},X_{2},\dots$ are a sequence
of i.i.d.\ random variables with $X_{1}\in[-2,2]$, $\mathbb{E}X_{1}=0$.
Take any $0<\delta<1$. Then there is an absolute constant $D_{0}$
such that with probability at least $1-\delta$, for all $n>0$ simultaneously, 

\[
\left|\sum_{i=1}^{n}X_{i}\right|\leq D_{0}\left(1+\ln\frac{1}{\delta}+\sqrt{4n\left(\left[\ln\ln\right]_{+}4n+\ln\frac{1}{\delta}\right)}\right)
\]
\end{prop}
In Algorithm~\ref{alg:The-active-learning}, we use CheckSignificant
to detect whether the expected number of queries that return label
0 at location $U_{k}$ ($V_{k}$) is more/less than the expected number
of label 1 with a statistical significance.

CheckSignificant-Var is based on the following uniform concentration
result which further utilizes the empirical variance $V_{n}=\frac{n}{n-1}\left(\sum_{i=1}^{n}X_{i}^{2}-\frac{1}{n}\left(\sum_{i=1}^{n}X_{i}\right)^{2}\right)$:
\begin{prop}
\label{prop:uniform-empirical-berstein}There is an absolute constant
$D_{1}$ such that with probability at least $1-\delta$, for all
$n\geq\ln\frac{1}{\delta}$ simultaneously, 

\[
\left|\sum_{i=1}^{n}X_{i}\right|\leq D_{1}\left(1+\ln\frac{1}{\delta}+\sqrt{\left(1+\ln\frac{1}{\delta}+V_{n}\right)\left(\left[\ln\ln\right]_{+}(1+\ln\frac{1}{\delta}+V_{n})+\ln\frac{1}{\delta}\right)}\right)
\]
\end{prop}
The use of variance results in a tighter bound when $\text{Var}(X_{i})$
is small.

In Algorithm~\ref{alg:The-active-learning}, we use CheckSignificant-Var
to detect the statistical significance of the relative order of the
number of queries that return non-abstention responses at $U_{k}$
($V_{k}$) compared to the number of non-abstention responses at $M_{k}$.
This results in a better query complexity than using CheckSignificant
under Condition~\ref{cond:oracle-shape}, since the variance of the
number of abstention responses approaches 0 when the interval $[L_{k},R_{k}]$
zooms in on $\theta^{*}$.\footnote{We do not apply CheckSignificant-Var to 0/1 labels, because unlike
the difference between the numbers of abstention responses at $U_{k}$
($V_{k}$) and $M_{k}$, the variance of the difference between the
numbers of 0 and 1 labels stays above a positive constant.}

\subsection{Analysis}

For Algorithm~\ref{alg:The-active-learning} to be statistically
consistent, we only need Condition~\ref{cond:oracle-correct}.
\begin{thm}
\label{thm:algo-correctness} Let $\theta^{*}$ be the ground truth.
If the labeler $L$ satisfies Condition~\ref{cond:oracle-correct}
and Algorithm~ \ref{alg:The-active-learning} stops to output $\hat{\theta}$,
then $\left|\theta^{*}-\hat{\theta}\right|\leq\epsilon$ with probability
at least $1-\frac{\delta}{2}$.
\end{thm}
Under additional Conditions~\ref{cond:oracle-noshape} and \ref{cond:oracle-shape},
we can derive upper bounds of the query complexity for our algorithm.
(Recall $f$ and $\beta$ are defined in Conditions~\ref{cond:oracle-noshape}
and \ref{cond:oracle-shape}.)
\begin{thm}
\label{thm:algo-label-noshape}Let $\theta^{*}$ be the ground truth,
and $\hat{\theta}$ be the output of Algorithm~\ref{alg:The-active-learning}.
Under Conditions~\ref{cond:oracle-correct} and \ref{cond:oracle-noshape},
with probability at least $1-\delta$, Algorithm~\ref{alg:The-active-learning}
makes at most $\tilde{O}\left(\frac{1}{f(\frac{\epsilon}{2})}\epsilon^{-2\beta}\right)$
queries.
\end{thm}

\begin{thm}
\label{thm:algo-label-shape}Let $\theta^{*}$ be the ground truth,
and $\hat{\theta}$ be the output of Algorithm~\ref{alg:The-active-learning}.
Under Conditions~\ref{cond:oracle-correct} and \ref{cond:oracle-shape},
with probability at least $1-\delta$, Algorithm \ref{alg:The-active-learning}
makes at most $\tilde{O}\left(\frac{1}{f(\frac{\epsilon}{2})}\right)$
queries.
\end{thm}
The query complexity given by Theorem~\ref{thm:algo-label-shape}
is independent of $\beta$ that decides the flipping rate, and consequently
smaller than the bound in Theorem~\ref{thm:algo-label-noshape}.
This improvement is due to the use of abstention responses, which
become much more informative under Condition~\ref{cond:oracle-shape}.

\subsection{Lower Bounds}

In this subsection, we give lower bounds of query complexity in the
one-dimensional case and establish near optimality of Algorithm~\ref{alg:The-active-learning}.
We will give corresponding lower bounds for the high-dimensional case
in the next section.

The lower bound in \cite{YCJ2015} can be easily generalized to Condition~\ref{cond:oracle-noshape}:
\begin{thm}
\label{thm:1d-lb-noshape} (\cite{YCJ2015}) There is a universal
constant $\delta_{0}\in(0,1)$ and a labeler $L$ satisfying Conditions~\ref{cond:oracle-correct}
and \ref{cond:oracle-noshape}, such that for any active learning
algorithm $\mathcal{A}$, there is a $\theta^{*}\in[0,1]$, such that
for small enough $\epsilon$, $\Lambda(\epsilon,\delta_{0},\mathcal{A},L,\theta^{*})\geq\Omega\left(\frac{1}{f(\epsilon)}\epsilon^{-2\beta}\right)$.
\end{thm}
Our query complexity (Theorem~\ref{thm:algo-label-shape}) for the
algorithm is also almost tight under Conditions~\ref{cond:oracle-correct}
and \ref{cond:oracle-shape} with a polynomial abstention rate.
\begin{thm}
\label{thm:1d-lb-shape}There is a universal constant $\delta_{0}\in(0,1)$
and a labeler $L$ satisfying Conditions~\ref{cond:oracle-correct},
\ref{cond:oracle-noshape}, and \ref{cond:oracle-shape} with $f(x)=C'x^{\alpha}$
($C'>0$ and $0<\alpha\leq2$ are constants), such that for any active
learning algorithm $\mathcal{A}$, there is a $\theta^{*}\in[0,1]$,
such that for small enough $\epsilon$, $\Lambda(\epsilon,\delta_{0},\mathcal{A},L,\theta^{*})\geq\Omega\left(\epsilon^{-\alpha}\right)$.
\end{thm}

\subsection{Remarks}

Our results confirm the intuition that learning with abstention is
easier than learning with noisy labels. This is true because a noisy
label might mislead the learning algorithm, but an abstention response
never does. Our analysis shows, in particular, that if the labeler
never abstains, and outputs completely noisy labels with probability
bounded by $1-\left|x-\theta^{*}\right|^{\gamma}$ (i.e., $P(Y\neq\mathbb{I}\left[x>\theta^{*}\right]\mid x)\leq\frac{1}{2}\left(1-\left|x-\theta^{*}\right|^{\gamma}\right)$),
then the near optimal query complexity of $\tilde{O}\left(\epsilon^{-2\gamma}\right)$
is significantly larger than the near optimal $\tilde{O}\left(\epsilon^{-\gamma}\right)$
query complexity associated with a labeler who only abstains with
probability $P(Y=\perp\mid x)\leq1-\left|x-\theta^{*}\right|^{\gamma}$
and never flips a label. More precisely, while in both cases the labeler
outputs the same amount of corrupted labels, the query complexity
of the abstention-only case is significantly smaller than the noise-only
case. 

Note that the query complexity of Algorithm~\ref{alg:The-active-learning}
consists of two kinds of queries: queries which return 0/1 labels
and are used by function CheckSignificant, and queries which return
abstention and are used by function CheckSignificant-Var. Algorithm~\ref{alg:The-active-learning}
will stop querying when the responses of one of the two kinds of queries
are statistically significant. Under Condition~\ref{cond:oracle-noshape},
our proof actually shows that the optimal number of queries is dominated
by the number of queries used by CheckSignificant function. In other
words, a simplified variant of Algorithm~\ref{alg:The-active-learning}
which excludes use of abstention feedback is near optimal. Similarly,
under Condition~\ref{cond:oracle-shape}, the optimal query complexity
is dominated by the number of queries used by CheckSignificant-Var
function. Hence the variant of Algorithm~\ref{alg:The-active-learning}
which disregards 0/1 labels would be near optimal.

\section{The multidimensional case\label{sec:d-gt-1}}

We follow \cite{CN08} to generalize the results from one-dimensional
thresholds to the d-dimensional $(d>1)$ smooth boundary fragment
class $\Sigma(K,\gamma)$.

\subsection{Lower bounds\label{sec:Lower-bound}}
\begin{thm}
\label{thm:hd-lb-noshape}There are universal constants $\delta_{0}\in(0,1)$,
$c_{0}>0$, and a labeler $L$ satisfying Conditions~\ref{cond:oracle-correct}
and \ref{cond:oracle-noshape}, such that for any active learning
algorithm $\mathcal{A}$, there is a $g^{*}\in\Sigma(K,\gamma)$,
such that for small enough $\epsilon$,  $\Lambda(\epsilon,\delta_{0},\mathcal{A},L,g^{*})\geq\Omega\left(\frac{1}{f(c_{0}\epsilon)}\epsilon^{-2\beta-\frac{d-1}{\gamma}}\right)$.
\end{thm}

\begin{thm}
\label{thm:hd-lb-shape}There is a universal constant $\delta_{0}\in(0,1)$
and a labeler $L$ satisfying Conditions~\ref{cond:oracle-correct},
\ref{cond:oracle-noshape}, and Condition~\ref{cond:oracle-shape}
with $f(x)=C'x^{\alpha}$ ($C'>0$ and $0<\alpha\leq2$ are constants),
such that for any active learning algorithm $\mathcal{A}$, there
is a $g^{*}\in\Sigma(K,\gamma)$, such that for small enough $\epsilon$,
$\Lambda(\epsilon,\delta_{0},\mathcal{A},L,g^{*})\geq\Omega\left(\epsilon^{-\alpha-\frac{d-1}{\gamma}}\right)$.
\end{thm}

\subsection{Algorithm and Analysis}

Recall the decision boundary of the smooth boundary fragment class
can be seen as the epigraph of a smooth function $[0,1]^{d-1}\rightarrow[0,1]$.
For $d>1$, we can reduce the problem to the one-dimensional problem
by discretizing the first $d-1$ dimensions of the instance space
and then perform a polynomial interpolation. The algorithm is shown
as Algorithm~\ref{alg:The-active-learning-hd}. For the sake of simplicity,
we assume $\gamma$, $M/\gamma$ in Algorithm~\ref{alg:The-active-learning-hd}
are integers.

\begin{algorithm}
\begin{algorithmic}[1]
\State{Input: $\delta$, $\epsilon$, $\gamma$}
\State{$M \gets \Theta\left(\epsilon^{-1/\gamma}\right)$. $\mathcal{L} \gets \left\{ \frac{0}{M}, \frac{1}{M}, \dots, \frac{M-1}{M} \right\}^{d-1} $}
\State{For each $l \in \mathcal{L}$, apply Algorithm~\ref{alg:The-active-learning}  with parameter ($\epsilon$, $\delta/M^{d-1}$) to learn a threshold $g_{l}$ that approximates $g^*(l)$} \label{line:hd:1d}
\State{Partition the instance space into cells $\left\{I_q\right\}$ indexed by $q \in \left\{ 0, 1, \dots, \frac{M}{\gamma}-1\right\}^{d-1}$, where 
	$$ I_q = \left[\frac{q_1\gamma}{M}, \frac{(q_1+1)\gamma}{M}\right] \times \cdots \times \left[\frac{q_{d-1}\gamma}{M}, \frac{(q_{d-1}+1)\gamma}{M}\right]$$}
\State{For each cell $I_q$, perform a polynomial interpolation:
	$ g_q(\tilde{\boldsymbol{x}}) = \sum_{l\in I_q \cap \mathcal{L}}g_{l}Q_{q,l}(\tilde{\boldsymbol{x}})$, where 
    $$ Q_{q,l}(\tilde{\boldsymbol{x}}) = \prod_{i=1}^{d-1} \prod_{j=0, j\neq Ml_i-\gamma q_i}^{\gamma} \frac{\tilde{\boldsymbol{x}}_{i}-(\gamma q_{i}+j)/M}{l_{i}-(\gamma q_{i}+j)/M}$$}
\State{Output: $g(\tilde{\boldsymbol{x}}) = \sum_{q \in \left\{ 0, 1, \dots, \frac{M}{\gamma}-1\right\}^{d-1}} g_q(\tilde{\boldsymbol{x}}) \mathds{1}\left[\tilde{\boldsymbol{x}} \in q\right]$}
\end{algorithmic}

\caption{\label{alg:The-active-learning-hd}The active learning algorithm for
the smooth boundary fragment class}
\end{algorithm}

We have similar consistency guarantee and upper bounds as in the one-dimensional
case.
\begin{thm}
\label{thm:algo-correctness-hd}Let $g^{*}$ be the ground truth.
If the labeler $L$ satisfies Condition~\ref{cond:oracle-correct}
and Algorithm~\ref{alg:The-active-learning-hd} stops to output $g$,
then $\left\Vert g^{*}-g\right\Vert \leq\epsilon$ with probability
at least $1-\frac{\delta}{2}$.
\end{thm}

\begin{thm}
\label{thm:algo-label-noshape--hd}Let $g^{*}$ be the ground truth,
and $g$ be the output of Algorithm~\ref{alg:The-active-learning-hd}.
Under Conditions~\ref{cond:oracle-correct} and \ref{cond:oracle-noshape},
with probability at least $1-\delta$, Algorithm~\ref{alg:The-active-learning-hd}
makes at most $\tilde{O}\left(\frac{d}{f\left(\epsilon/2\right)}\epsilon^{-2\beta-\frac{d-1}{\gamma}}\right)$
queries.
\end{thm}

\begin{thm}
\label{thm:algo-label-shape-hd}Let $g^{*}$ be the ground truth,
and $g$ be the output of Algorithm~\ref{alg:The-active-learning-hd}.
Under Conditions~\ref{cond:oracle-correct} and \ref{cond:oracle-shape},
with probability at least $1-\delta$, Algorithm~\ref{alg:The-active-learning-hd}
makes at most $\tilde{O}\left(\frac{d}{f\left(\epsilon/2\right)}\epsilon^{-\frac{d-1}{\gamma}}\right)$
queries.
\end{thm}

\paragraph*{Acknowledgments.}

We thank NSF under IIS-1162581, CCF-1513883, and CNS-1329819 for research
support. 

\bibliographystyle{plain}
\bibliography{additional,lpactive}

\begin{thebibliography}{10}

\bibitem{BL13}
M.-F. Balcan and P.~M. Long.
\newblock Active and passive learning of linear separators under log-concave
  distributions.
\newblock In {\em COLT}, 2013.

\bibitem{BBL06}
Maria-Florina Balcan, Alina Beygelzimer, and John Langford.
\newblock Agnostic active learning.
\newblock In {\em Proceedings of the 23rd international conference on Machine
  learning}, pages 65--72. ACM, 2006.

\bibitem{balcan2012robust}
Maria-Florina Balcan and Steve Hanneke.
\newblock Robust interactive learning.
\newblock In {\em Proceedings of The 25th Conference on Learning Theory}, 2012.

\bibitem{BHLZ10}
A.~Beygelzimer, D.~Hsu, J.~Langford, and T.~Zhang.
\newblock Agnostic active learning without constraints.
\newblock In {\em NIPS}, 2010.

\bibitem{beygelzimer2016search}
Alina Beygelzimer, Daniel Hsu, John Langford, and Chicheng Zhang.
\newblock Search improves label for active learning.
\newblock {\em arXiv preprint arXiv:1602.07265}, 2016.

\bibitem{CN08}
Rui~M. Castro and Robert~D. Nowak.
\newblock Minimax bounds for active learning.
\newblock {\em {IEEE} Transactions on Information Theory}, 54(5):2339--2353,
  2008.

\bibitem{chen2015sequential}
Yuxin Chen, S~Hamed Hassani, Amin Karbasi, and Andreas Krause.
\newblock Sequential information maximization: When is greedy near-optimal?
\newblock In {\em Proceedings of The 28th Conference on Learning Theory}, pages
  338--363, 2015.

\bibitem{CAL94}
D.~A. Cohn, L.~E. Atlas, and R.~E. Ladner.
\newblock Improving generalization with active learning.
\newblock {\em Machine Learning}, 15(2), 1994.

\bibitem{D05}
S.~Dasgupta.
\newblock Coarse sample complexity bounds for active learning.
\newblock In {\em NIPS}, 2005.

\bibitem{DHM07}
S.~Dasgupta, D.~Hsu, and C.~Monteleoni.
\newblock A general agnostic active learning algorithm.
\newblock In {\em NIPS}, 2007.

\bibitem{fang2012don}
Meng Fang and Xingquan Zhu.
\newblock I don't know the label: Active learning with blind knowledge.
\newblock In {\em Pattern Recognition (ICPR), 2012 21st International
  Conference on}, pages 2238--2241. IEEE, 2012.

\bibitem{hanneke2007teaching}
Steve Hanneke.
\newblock Teaching dimension and the complexity of active learning.
\newblock In {\em Learning Theory}, pages 66--81. Springer, 2007.

\bibitem{hegedHus1995generalized}
Tibor Heged{\H{u}}s.
\newblock Generalized teaching dimensions and the query complexity of learning.
\newblock In {\em Proceedings of the eighth annual conference on Computational
  learning theory}, pages 108--117. ACM, 1995.

\bibitem{K06}
M.~K{\"a}{\"a}ri{\"a}inen.
\newblock Active learning in the non-realizable case.
\newblock In {\em ALT}, 2006.

\bibitem{kading2015active}
Christoph Kading, Alexander Freytag, Erik Rodner, Paul Bodesheim, and Joachim
  Denzler.
\newblock Active learning and discovery of object categories in the presence of
  unnameable instances.
\newblock In {\em Computer Vision and Pattern Recognition (CVPR), 2015 IEEE
  Conference on}, pages 4343--4352. IEEE, 2015.

\bibitem{li2013some}
Yuan-Chuan Li and Cheh-Chih Yeh.
\newblock Some equivalent forms of bernoulli's inequality: A survey.
\newblock {\em Applied Mathematics}, 4(07):1070, 2013.

\bibitem{minsker2012plug}
Stanislav Minsker.
\newblock Plug-in approach to active learning.
\newblock {\em Journal of Machine Learning Research}, 13(Jan):67--90, 2012.

\bibitem{NJC15}
Mohammad Naghshvar, Tara Javidi, and Kamalika Chaudhuri.
\newblock Bayesian active learning with non-persistent noise.
\newblock {\em {IEEE} Transactions on Information Theory}, 61(7):4080--4098,
  2015.

\bibitem{N11}
R.~D. Nowak.
\newblock The geometry of generalized binary search.
\newblock {\em IEEE Transactions on Information Theory}, 57(12):7893--7906,
  2011.

\bibitem{raginsky2011lower}
Maxim Raginsky and Alexander Rakhlin.
\newblock Lower bounds for passive and active learning.
\newblock In {\em Advances in Neural Information Processing Systems}, pages
  1026--1034, 2011.

\bibitem{ramdas2015sequential}
Aaditya Ramdas and Akshay Balsubramani.
\newblock Sequential nonparametric testing with the law of the iterated
  logarithm.
\newblock In {\em Proceedings of the Conference on Uncertainty in Artificial
  Intelligence}, 2016.

\bibitem{T04}
A.~B. Tsybakov.
\newblock {Optimal aggregation of classifiers in statistical learning}.
\newblock {\em Annals of Statistics}, 32:135--166, 2004.

\bibitem{urner2012learning}
Ruth Urner, Shai Ben-david, and Ohad Shamir.
\newblock Learning from weak teachers.
\newblock In {\em International Conference on Artificial Intelligence and
  Statistics}, pages 1252--1260, 2012.

\bibitem{YCJ2015}
Songbai Yan, Kamalika Chaudhuri, and Tara Javidi.
\newblock Active learning from noisy and abstention feedback.
\newblock In {\em Communication, Control, and Computing (Allerton), 2015 53th
  Annual Allerton Conference on}. IEEE, 2015.

\bibitem{ZC14}
Chicheng Zhang and Kamalika Chaudhuri.
\newblock Beyond disagreement-based agnostic active learning.
\newblock In {\em Advances in Neural Information Processing Systems}, pages
  442--450, 2014.

\bibitem{zhang2015active}
Chicheng Zhang and Kamalika Chaudhuri.
\newblock Active learning from weak and strong labelers.
\newblock In {\em Advances in Neural Information Processing Systems}, pages
  703--711, 2015.

\end{thebibliography}

\newpage{}

\appendix

\section{Proof of query complexities}

\subsection{\label{subsec:The-properties-of-testing}Properties of adaptive sequential
testing in Procedure \ref{alg:test}}
\begin{lem}
\label{lem:CheckSignificant-correctness}Suppose $\left\{ X_{i}\right\} _{i=1}^{\infty}$
is a sequence of i.i.d.\ random variables such that $\mathbb{E}X_{i}\leq0$,
$\left|X_{i}\right|\leq1$. Let $\delta>0$. Then with probability
at least $1-\delta$, for all $n\in\mathbb{N}$ simultaneously CheckSignificant$\left(\left\{ X_{i}\right\} _{i=1}^{n},\delta\right)$
in Procedure~\ref{alg:test} returns false.
\end{lem}
\begin{proof}
This is immediate by applying Proposition~\ref{prop:uniform-berstein}
to $X_{i}-\mathbb{E}X_{i}$.
\end{proof}
\begin{lem}
\label{lem:CheckSignificant-effectiveness}Suppose $\left\{ X_{i}\right\} _{i=1}^{\infty}$
is a sequence of i.i.d.\ random variables such that $\mathbb{E}X_{i}>\epsilon>0$,
$\left|X_{i}\right|\leq1$. Let $\delta\in[0,\frac{1}{3}]$, $N\geq\frac{\xi}{\epsilon^{2}}\ln\frac{1}{\delta}[\ln\ln]_{+}\frac{1}{\epsilon}$
($\xi$ is an absolute constant specified in the proof). Then with
probability at least $1-\delta$, CheckSignificant$\left(\left\{ X_{i}\right\} _{i=1}^{N},\delta\right)$
in Procedure~\ref{alg:test} returns true.
\end{lem}
\begin{proof}
Let $S_{N}=\sum_{i=1}^{N}X_{i}$. CheckSignificant$\left(\left\{ X_{i}\right\} _{i=1}^{N},\delta\right)$
returns false if and only if \\
$S_{N}\leq D_{0}\left(1+\ln\frac{1}{\delta}+\sqrt{N\left([\ln\ln]_{+}N+\ln\frac{1}{\delta}\right)}\right)$.

\begin{align*}
 & \Pr\left(S_{N}\leq D_{0}\left(1+\ln\frac{1}{\delta}+\sqrt{N\left([\ln\ln]_{+}N+\ln\frac{1}{\delta}\right)}\right)\right)\\
\leq & \Pr\left(S_{N}\leq D_{0}\left(1+\ln\frac{1}{\delta}+\sqrt{N[\ln\ln]_{+}N}+\sqrt{N\ln\frac{1}{\delta}}\right)\right)\\
\leq & \Pr\left(S_{N}-N\mathbb{E}X_{i}\leq D_{0}\left(1+\ln\frac{1}{\delta}+\sqrt{N[\ln\ln]_{+}N}+\sqrt{N\ln\frac{1}{\delta}}\right)-N\epsilon\right)
\end{align*}

Suppose $N=\frac{c\xi}{\epsilon^{2}}\ln\frac{1}{\delta}[\ln\ln]_{+}\frac{1}{\epsilon}$
for constant $c\geq1$ and $\xi$. $\xi$ is set to be sufficiently
large, such that (1) $\xi\geq4D_{0}^{2}$; (2) $\frac{2D_{0}}{\sqrt{\xi}}+D_{0}\left(3+\sqrt{[\ln\ln]_{+}\xi}\right)+D_{0}-\sqrt{\xi}/2\leq-\sqrt{\frac{1}{2}}$;
(3) $f(x)=D_{0}\sqrt{[\ln\ln]_{+}x}-\sqrt{x}/2$ is decreasing when
$x>\xi$. Here (2) is satisfiable since $\frac{D_{0}}{\sqrt{\xi}}+D_{0}\sqrt{[\ln\ln]_{+}\xi}-\sqrt{\xi}/2\rightarrow-\infty$
as $\xi\rightarrow\infty$, (3) is satisfiable since $f'(x)\rightarrow-\infty$
as $x\rightarrow\infty$. (2) and (3) together implies $\frac{2D_{0}}{\sqrt{\xi}}+D_{0}\left(3+\sqrt{[\ln\ln]_{+}c\xi}\right)+D_{0}-\sqrt{c\xi}/2\leq-\sqrt{\frac{1}{2}}$.

\begin{align*}
 & \frac{1}{\sqrt{N}}\left(D_{0}\left(1+\ln\frac{1}{\delta}+\sqrt{N[\ln\ln]_{+}N}+\sqrt{N\ln\frac{1}{\delta}}\right)-N\epsilon\right)\\
= & \sqrt{\ln\frac{1}{\delta}}\left(\frac{D_{0}\epsilon(1+\ln\frac{1}{\delta})}{\sqrt{c\xi[\ln\ln]_{+}\frac{1}{\epsilon}}\ln\frac{1}{\delta}}+D_{0}\sqrt{\frac{[\ln\ln]_{+}\left(\frac{c\xi}{\epsilon^{2}}\ln\frac{1}{\delta}[\ln\ln]_{+}\frac{1}{\epsilon}\right)}{\ln\frac{1}{\delta}}}+D_{0}-\sqrt{c\xi[\ln\ln]_{+}\frac{1}{\epsilon}}\right)
\end{align*}

Since $[\ln\ln]_{+}\frac{1}{\epsilon},c,\ln\frac{1}{\delta}\geq1$
and $\epsilon<1$, we have $\frac{D_{0}\epsilon(1+\ln\frac{1}{\delta})}{\sqrt{c\xi[\ln\ln]_{+}\frac{1}{\epsilon}}\ln\frac{1}{\delta}}\leq\frac{2D_{0}}{\sqrt{\xi}}$. 

Since $[\ln\ln]_{+}x\geq1$ if $x\ge1$, we have $[\ln\ln]_{+}\frac{1}{\epsilon}\leq\frac{1}{\epsilon}$,
and thus

\begin{eqnarray*}
\sqrt{[\ln\ln]_{+}\left(\frac{c\xi}{\epsilon^{2}}\ln\frac{1}{\delta}[\ln\ln]_{+}\frac{1}{\epsilon}\right)} & = & \sqrt{\ln\left[\max\left\{ e,2\ln\frac{1}{\epsilon}+\ln c\xi+\ln\ln\frac{1}{\delta}+\ln[\ln\ln]_{+}\frac{1}{\epsilon}\right\} \right]}\\
 & \leq & \sqrt{\ln\left[\max\left\{ e,3\ln\frac{1}{\epsilon}+\ln c\xi+[\ln\ln]_{+}\frac{1}{\delta}\right\} \right]}\\
 & \overset{(a)}{\leq} & \sqrt{\ln\left[\max\left\{ e,9\ln\frac{1}{\epsilon}\ln c\xi[\ln\ln]_{+}\frac{1}{\delta}\right\} \right]}\\
 & \leq & \sqrt{3+[\ln\ln]_{+}\frac{1}{\epsilon}+[\ln\ln]_{+}c\xi+\ln[\ln\ln]_{+}\frac{1}{\delta}}\\
 & \overset{(b)}{\leq} & \sqrt{3}+\sqrt{[\ln\ln]_{+}c\xi}+\sqrt{[\ln\ln]_{+}\frac{1}{\epsilon}}+\sqrt{\ln[\ln\ln]_{+}\frac{1}{\delta}}
\end{eqnarray*}

where (a) follows by $a+b+c\leq3abc$ if $a,b,c\ge1$, and (b) follows
by $\sqrt{\sum_{i}x_{i}}\leq\sum_{i}\sqrt{x_{i}}$ if $x_{i}\geq0$.

Thus, we have
\begin{align*}
 & \frac{1}{\sqrt{N}}\left(D_{0}\left(1+\ln\frac{1}{\delta}+\sqrt{N[\ln\ln]_{+}N}+\sqrt{N\ln\frac{1}{\delta}}\right)-N\epsilon\right)\\
\leq & \sqrt{\ln\frac{1}{\delta}}\left(\frac{2D_{0}}{\sqrt{\xi}}+D_{0}\frac{\sqrt{3}+\sqrt{[\ln\ln]_{+}c\xi}+\sqrt{[\ln\ln]_{+}\frac{1}{\epsilon}}+\sqrt{\ln[\ln\ln]_{+}\frac{1}{\delta}}}{\sqrt{\ln\frac{1}{\delta}}}+D_{0}-\sqrt{c\xi[\ln\ln]_{+}\frac{1}{\epsilon}}\right)\\
\overset{(c)}{\leq} & \sqrt{\ln\frac{1}{\delta}}\left(\frac{2D_{0}}{\sqrt{\xi}}+D_{0}\left(3+\sqrt{[\ln\ln]_{+}c\xi}\right)+D_{0}-\sqrt{c\xi}/2\right)\\
\overset{(d)}{\leq} & -\sqrt{\ln\frac{1}{\delta}/2}
\end{align*}

(c) follows by $\sqrt{\ln\frac{1}{\delta}}\geq\max\left\{ 1,\sqrt{\ln[\ln\ln]_{+}\frac{1}{\delta}}\right\} $,
$D_{0}\ge1$, and $\sqrt{[\ln\ln]_{+}\frac{1}{\epsilon}}\left(\frac{D_{0}}{\sqrt{\ln\frac{1}{\delta}}}-\sqrt{c\xi}\right)\leq D_{0}-\sqrt{c\xi}\leq-\sqrt{c\xi}/2$
if $c\xi\geq4D_{0}^{2}$. (d) follows by our choose of $\xi$.

Therefore, 

\begin{align*}
 & \Pr\left(S_{N}-N\mathbb{E}X_{i}\leq D_{0}\left(1+\ln\frac{1}{\delta}+\sqrt{N[\ln\ln]_{+}N}+\sqrt{N\ln\frac{1}{\delta}}\right)-N\epsilon\right)\\
\leq & \Pr\left(S_{N}-N\mathbb{E}X_{i}\leq-\sqrt{N\ln\frac{1}{\delta}/2}\right)
\end{align*}

which is at most $\delta$ by Hoeffding Bound.
\end{proof}
\begin{lem}
\label{lem:CheckSignificant-var-correctness}Suppose $\left\{ X_{i}\right\} _{i=1}^{\infty}$
is a sequence of i.i.d.\ random variables such that $\mathbb{E}X_{i}\leq0$,
$\left|X_{i}\right|\leq1$. Let $\delta>0$. Then with probability
at least $1-\delta$, for all $n$ simultaneously CheckSignificant-Var$\left(\left\{ X_{i}\right\} _{i=1}^{n},\delta\right)$
in Procedure~\ref{alg:test} returns false.
\end{lem}
\begin{proof}
Define $Y_{i}=X_{i}-\mathbb{E}X_{i}$. It is easy to check $\frac{n}{n-1}\left(\sum_{i=1}^{n}Y_{i}^{2}-\frac{1}{n}\left(\sum_{i=1}^{n}Y_{i}\right)^{2}\right)=\frac{n}{n-1}\left(\sum_{i=1}^{n}X_{i}^{2}-\frac{1}{n}\left(\sum_{i=1}^{n}X_{i}\right)^{2}\right)$.
The result is immediate from Proposition~\ref{prop:uniform-empirical-berstein}.
\end{proof}
\begin{lem}
\label{lem:CheckSignificant-var-effectiveness}Suppose $\left\{ X_{i}\right\} _{i=1}^{\infty}$
is a sequence of i.i.d.\ random variables such that $\mathbb{E}X_{i}>\tau\epsilon$,
$\left|X_{i}\right|\leq1$, $\text{Var}\left(X_{i}\right)\leq2\epsilon$
where $0<\epsilon\leq1$, $\tau>0$. Let $\delta<1$, $N=\frac{\xi}{\tau\epsilon}\ln\frac{2}{\delta}$
($\xi$ is a constant specified in the proof). Then with probability
at least $1-\delta$, CheckSignificant-Var$\left(\left\{ X_{i}\right\} _{i=1}^{N},\delta\right)$
in Procedure~\ref{alg:test} returns true.
\end{lem}
\begin{proof}
Let $Y_{i}=X_{i}-\mathbb{E}X_{i}$, $\eta$ be the constant $\eta$
in Lemma~\ref{lem:diff in E}. Set $\xi=\max(\eta,\frac{16}{\tau}+\frac{8}{3})$. 

CheckSignificant-Var$\left(\left\{ X_{i}\right\} _{i=1}^{N},\delta\right)$
returns false if and only if $\sum_{i=1}^{N}X_{i}\leq\ensuremath{q(N,\text{Var},\delta)}$.

By applying Lemma~\ref{lem:diff in E} to $X_{i}$, $\frac{\ensuremath{q(N,\text{Var},\delta)}}{N}-\mathbb{E}X_{i}\leq-\tau\epsilon/2$
with probability at least $1-\delta/2$. 

Applying Bernstein's inequality to $Y_{i}$, we have 

\begin{eqnarray*}
\Pr\left(\frac{1}{N}\sum_{i=1}^{N}Y_{i}\leq-\tau\epsilon/2\right) & \leq & \exp\left(-\frac{N\left(-\tau\epsilon\right)^{2}/4}{4\epsilon+2\tau\epsilon/3}\right)\\
 & = & \exp\left(-\frac{\xi\ln\frac{2}{\delta}}{16/\tau+8/3}\right)\\
 & \leq & \delta/2
\end{eqnarray*}

Thus, by a union bound, 
\begin{align*}
 & \Pr\left(\sum_{i=1}^{N}X_{i}\leq\ensuremath{q(N,\text{Var},\delta)}\right)\\
\leq & \Pr\left(\frac{\ensuremath{q(N,\text{Var},\delta)}}{N}-\mathbb{E}X_{i}\geq-\tau\epsilon/2\right)\\
 & +\Pr\left(\frac{\ensuremath{q(N,\text{Var},\delta)}}{N}-\mathbb{E}X_{i}\leq-\tau\epsilon/2\text{ and }\frac{1}{N}\sum_{i=1}^{N}X_{i}\leq\frac{\ensuremath{q(N,\text{Var},\delta)}}{N}\right)\\
\leq & \delta/2+\Pr\left(\frac{\ensuremath{q(N,\text{Var},\delta)}}{N}-\mathbb{E}X_{i}\leq-\tau\epsilon/2\text{ and }\frac{1}{N}\sum_{i=1}^{N}Y_{i}\leq\frac{\ensuremath{q(n,\text{Var},\delta)}}{N}-\mathbb{E}X_{i}\right)\\
\leq & \delta/2+\Pr\left(\frac{1}{N}\sum_{i=1}^{N}Y_{i}\leq-\tau\epsilon/2\right)\\
\leq & \delta
\end{align*}
\end{proof}

\subsection{The one-dimensional case}
\begin{proof}[Proof of Theorem \ref{thm:algo-correctness}]
 \label{lem:(correctness)} Since $\hat{\theta}=\left(L_{\log\frac{1}{2\epsilon}}+R_{\log\frac{1}{2\epsilon}}\right)/2$
and $R_{\log\frac{1}{2\epsilon}}-L_{\log\frac{1}{2\epsilon}}=2\epsilon$,
$\left|\hat{\theta}-\theta^{*}\right|>\epsilon$ is equivalent to
$\theta^{*}\notin[L_{\log\frac{1}{2\epsilon}},R_{\log\frac{1}{2\epsilon}}]$.
We have 
\begin{eqnarray*}
\Pr\left(\left|\hat{\theta}-\theta^{*}\right|>\epsilon\right) & = & \Pr\left(\theta^{*}\notin[L_{\log\frac{1}{2\epsilon}},R_{\log\frac{1}{2\epsilon}}]\right)\\
 & = & \Pr\left(\exists k:\theta^{*}\in[L_{k},R_{k}]\text{ and }\theta^{*}\notin[L_{k+1},R_{k+1}]\right)\\
 & \le & \sum_{k=0}^{\log\frac{1}{2\epsilon}-1}\Pr\left(\theta^{*}\in[L_{k},R_{k}]\text{ and }\theta^{*}\notin[L_{k+1},R_{k+1}]\right)
\end{eqnarray*}

For any $k=0,\dots,\log\frac{1}{2\epsilon}-1$, define $\mathbb{Q}_{k}=\left\{ (p,q):p,q\in\mathbb{Q}\cap[0,1]\text{ and }q-p=\left(\frac{3}{4}\right)^{k}\right\} $
where $\mathbb{Q}$ is the set of rational numbers. Note that $L_{k},R_{k}\in\mathbb{Q}_{k}$,
and $\mathbb{Q}$ is countable. So we have

\begin{align*}
 & \Pr\left(\theta^{*}\in[L_{k},R_{k}]\text{ and }\theta^{*}\notin[L_{k+1},R_{k+1}]\right)\\
= & \sum_{(p,q)\in\mathbb{Q}_{k}:p\leq\theta^{*}\leq q}\Pr\left(L_{k}=p,R_{k}=q\text{ and }\theta^{*}\notin[L_{k+1},R_{k+1}]\right)\\
= & \sum_{(p,q)\in\mathbb{Q}_{k}:p\leq\theta^{*}\leq q}\Pr\left(\theta^{*}\notin[L_{k+1},R_{k+1}]|L_{k}=p,R_{k}=q\right)\Pr\left(L_{k}=p,R_{k}=q\right)
\end{align*}

Define event $E_{k,p,q}$ to be the event $L_{k}=p,R_{k}=q$. To show
$\Pr\left(\left|\hat{\theta}-\theta^{*}\right|>\epsilon\right)\leq\frac{\delta}{2}$,
it suffices to show $\Pr\left(\theta^{*}\notin[L_{k+1},R_{k+1}]|E_{k,p,q}\right)\leq\frac{\delta}{2\log\frac{1}{2\epsilon}}$
for any $k=0,\dots,\log\frac{1}{2\epsilon}-1$, $(p,q)\in\mathbb{Q}_{k}\text{ and }p\leq\theta^{*}\leq q$.

Conditioning on event $E_{k,p,q}$, event $\theta^{*}\notin[L_{k+1},R_{k+1}]$
happens only if some calls of CheckSignificant and CheckSignificant-Var
between Line \ref{line:algo-1d:test-start} and \ref{line:algo-1d:test-end}
of Algorithm~\ref{alg:The-active-learning} return true incorrectly.
In other words, at least one of following events happens for some
$n$: 

\begin{itemize}
\item $O_{k,p,q}^{(1)}$: $\theta^{*}\in[L_{k},U_{k}]$ and CheckSignificant-Var($\left\{ A_{i}^{(u)}-A_{i}^{(m)}\right\} _{i=1}^{n},\frac{\delta}{4\log\frac{1}{2\epsilon}}$)
returns true;
\item $O_{k,p,q}^{(2)}$: $\theta^{*}\in[V_{k},R_{k}]$ and CheckSignificant-Var($\left\{ A_{i}^{(v)}-A_{i}^{(m)}\right\} _{i=1}^{n},\frac{\delta}{4\log\frac{1}{2\epsilon}}$)
returns true;
\item $O_{k,p,q}^{(3)}$: $\theta^{*}\in[L_{k},U_{k}]$ and CheckSignificant($\left\{ -B_{i}^{(u)}\right\} _{i=1}^{n},\frac{\delta}{4\log\frac{1}{2\epsilon}}$)
returns true;
\item $O_{k,p,q}^{(4)}$: $\theta^{*}\in[V_{k},R_{k}]$ and CheckSignificant($\left\{ B_{i}^{(v)}\right\} _{i=1}^{n},\frac{\delta}{4\log\frac{1}{2\epsilon}}$)
returns true;
\end{itemize}
Note that since $[U_{k},V_{k}]\subset[L_{k+1},R_{k+1}]$ for any $k$
by our construction, if $\theta^{*}\in[U_{k},V_{k}]$ then $\theta^{*}\in[L_{k+1},R_{k+1}]$.
Besides, event $\theta^{*}\in[L_{k},U_{k}]$ and event $\theta^{*}\in[V_{k},R_{k}]$
are mutually exclusive.

Conditioning on event $E_{k,p,q}$, suppose for now $\theta^{*}\in[L_{k},U_{k}]$.

\begin{align*}
 & \Pr\left(O_{k,p,q}^{(1)}\mid E_{k,p,q}\right)\\
= & \Pr\left(\exists n:\text{CheckSignificant-Var(}\left\{ D_{i}^{(u,m)}\right\} _{i=1}^{n},\frac{\delta}{4\log\frac{1}{2\epsilon}}\text{) returns true}\mid\theta^{*}\in[L_{k},U_{k}],E_{k,p,q}\right)
\end{align*}

On event $\theta^{*}\in[L_{k},U_{k}]$ and $E_{k,p,q}$, the sequences
$\left\{ A_{i}^{(u)}\right\} $ and $\left\{ A_{i}^{(m)}\right\} $
are i.i.d., and $\mathbb{E}\Bigl[A_{i}^{(u)}-A_{i}^{(m)}\mid\theta^{*}\in[L_{k},U_{k}],E_{k,p,q}\Bigr]\leq0$.
By Lemma~\ref{lem:CheckSignificant-var-correctness}, the probability
above is at most $\frac{\delta}{4\log\frac{1}{2\epsilon}}$. 

Likewise, 
\begin{align*}
 & \Pr\left(O_{k,p,q}^{(3)}\mid E_{k,p,q}\right)\\
= & \Pr\left(\exists n:\text{CheckSignificant(}\left\{ -B_{i}^{(u)}\right\} _{i=1}^{n},\frac{\delta}{4\log\frac{1}{2\epsilon}}\text{) returns true}\mid\theta^{*}\in[L_{k},U_{k}],E_{k,p,q}\right)
\end{align*}

On event $\theta^{*}\in[L_{k},U_{k}]$ and $E_{k,p,q}$, the sequence
$\left\{ B_{i}^{(u)}\right\} $ is i.i.d., and $\mathbb{E}\left[-B_{i}^{(u)}\mid\theta^{*}\in[L_{k},U_{k}],E_{k,p,q}\right]\leq0$.
By Lemma~\ref{lem:CheckSignificant-correctness}, the probability
above is at most $\frac{\delta}{4\log\frac{1}{2\epsilon}}$. 

Thus, $\Pr\left(\theta^{*}\notin[L_{k+1},R_{k+1}]\mid E_{k,p,q}\right)\leq\frac{\delta}{2\log\frac{1}{2\epsilon}}$
when $\theta^{*}\in[L_{k},U_{k}]$. Similarly, when $\theta^{*}\in[V_{k},R_{k}]$,
we can show $\Pr\left(\theta^{*}\notin[L_{k+1},R_{k+1}]\mid E_{k,p,q}\right)\leq\Pr\left(O_{k,p,q}^{(2)}\mid E_{k,p,q}\right)+\Pr\left(O_{k,p,q}^{(4)}\mid E_{k,p,q}\right)\leq\frac{\delta}{2\log\frac{1}{2\epsilon}}$.

Therefore, $\Pr\left(\theta^{*}\notin[L_{k+1},R_{k+1}]\mid E_{k,p,q}\right)\leq\frac{\delta}{2\log\frac{1}{2\epsilon}}$,
and thus $\Pr\left(\left|\hat{\theta}-\theta^{*}\right|>\epsilon\right)\leq\delta/2$.
\end{proof}

\begin{proof}[Proof of Theorem \ref{thm:algo-label-noshape}]
Define $T_{k}$ to be the number of iterations of the loop at Line
\ref{line:algo-1d:loop-n}, $T=\sum_{k=0}^{\log\frac{1}{2\epsilon}-1}T_{k}$.
For any numbers $m_{1},m_{2},\dots,m_{\log\frac{1}{2\epsilon}-1}$,
we have:\\
\begin{eqnarray}
\Pr\left(T\geq m\right) & \leq & \Pr\left(\left|\hat{\theta}-\theta^{*}\right|>\epsilon\right)+\Pr\left(\left|\hat{\theta}-\theta^{*}\right|<\epsilon\text{ and }T\geq\sum_{k=0}^{\log\frac{1}{2\epsilon}-1}m_{k}\right)\nonumber \\
 & \leq & \frac{\delta}{2}+\Pr\left(T\geq\sum_{k=0}^{\log\frac{1}{2\epsilon}-1}m_{k}\text{ and }\left|\hat{\theta}-\theta^{*}\right|<\epsilon\right)\label{eq:label-decompose}\\
 & \leq & \frac{\delta}{2}+\sum_{k=0}^{\log\frac{1}{2\epsilon}-1}\Pr\left(T_{k}\geq m_{k}\text{ and }\left|\hat{\theta}-\theta^{*}\right|<\epsilon\right)\nonumber \\
 & \leq & \frac{\delta}{2}+\sum_{k=0}^{\log\frac{1}{2\epsilon}-1}\Pr\left(T_{k}\geq m_{k}\text{ and }\theta^{*}\in[L_{k},R_{k}]\right)\nonumber 
\end{eqnarray}

The first and the third inequality follows by union bounds. The second
follows by Theorem~\ref{thm:algo-correctness}. The last follows
since $\left|\hat{\theta}-\theta^{*}\right|<\epsilon$ is equivalent
to $\theta^{*}\in[L_{\log\frac{1}{2\epsilon}},R_{\log\frac{1}{2\epsilon}}]$,
which implies $\theta^{*}\in[L_{k},R_{k}]$ for all $k=0,\dots,\log\frac{1}{2\epsilon}-1$.

We define $\mathbb{Q}_{k}$ as in the previous proof. For all $k=0,\dots,\log\frac{1}{2\epsilon}-1$,

\begin{align*}
 & \Pr\left(T_{k}\geq m_{k}\text{ and }\theta^{*}\in[L_{k},R_{k}]\right)\\
= & \sum_{(p,q)\in\mathbb{Q}_{k}:p\leq\theta^{*}\leq q}\Pr\left(T_{k}\geq m_{k},L_{k}=p,R_{k}=q\right)\\
= & \sum_{(p,q)\in\mathbb{Q}_{k}:p\leq\theta^{*}\leq q}\Pr\left(T_{k}\geq m_{k}|L_{k}=p,R_{k}=q\right)\Pr\left(L_{k}=p,R_{k}=q\right)
\end{align*}

Thus, in order to prove the query complexity of Algorithm~\ref{alg:The-active-learning}
is $O\left(\sum_{k=0}^{\log\frac{1}{2\epsilon}-1}m_{k}\right)$, it
suffices to show that $\Pr\left(T_{k}\geq m_{k}\mid L_{k}=p,R_{k}=q\right)\leq\frac{\delta}{2\log\frac{1}{2\epsilon}}$
for any $k=0,\dots,\log\frac{1}{2\epsilon}-1$, $(p,q)\in\mathbb{Q}_{k}\text{ and }p\leq\theta^{*}\leq q$.

For each $k,p,q$, define event $E_{k,p,q}$ to be the event $L_{k}=p,R_{k}=q$.
Define $l_{k}=q-p=\left(\frac{3}{4}\right)^{k}$, $N_{k}$ to be $\tilde{\Theta}\left(\frac{1}{f(l_{k}/4)}l_{k}^{-2\beta}\right)$.
The logarithm factor of $N_{k}$ is to be specified later. Define
$S_{n}^{(u)}$ and $S_{n}^{(v)}$ to be the size of array $B^{(u)}$
and $B^{(v)}$ before Line \ref{line:algo-1d:test-start} respectively.

To show $\Pr\left(T_{k}\geq N_{k}\mid E_{k,p,q}\right)\leq\frac{\delta}{2\log\frac{1}{2\epsilon}}$,
it suffices to show that on event $E_{k,p,q}$, with probability at
least $1-\frac{\delta}{2\log\frac{1}{2\epsilon}}$, if $n=N_{k}$
then at least one of the two calls to CheckSignificant between Line
\ref{line:algo-1d:test-label-start} and Line \ref{line:algo-1d:test-end}
will return true.

On event $E_{k,p,q}$, if $\theta^{*}\in[L_{k},M_{k}]$ (note that
on event $E_{k,p,q}$, $L_{k}$ and $M_{k}$ are deterministic), then
$\left|V_{k}-\theta^{*}\right|\geq\frac{l_{k}}{4}$. We will show
\[
p_{1}:=\Pr\left(\text{CheckSignificant}\left(\left\{ B_{i}^{(v)}\right\} _{i=1}^{S_{N_{k}}^{(v)}},\frac{\delta}{4\log\frac{1}{2\epsilon}}\right)\text{ returns false}\mid E_{k,p,q}\right)\leq\frac{\delta}{2\log\frac{1}{2\epsilon}}
\]

To prove this, we will first show that $S_{N_{k}}^{(v)}$, the length
of the array $B^{(v)}$, is large with high probability, and then
apply Lemma~\ref{lem:CheckSignificant-effectiveness} to show that
CheckSignificant will return true if $S_{N_{k}}^{(v)}$ is large.

By definition, $S_{N_{k}}^{(v)}=\sum_{i=1}^{N_{k}}A_{i}^{(v)}$. By
Condition~\ref{cond:oracle-noshape}, $\mathbb{E}\left[A_{i}^{(v)}\mid E_{k,p,q}\right]=\Pr\left(Y\neq\perp\mid X=V_{k},E_{k,p,q}\right)\ge f\left(\frac{l_{k}}{4}\right)$. 

On event $E_{k,p,q}$, $\left\{ A_{i}^{(v)}\right\} $ is a sequence
of i.i.d.\ random variables. By the multiplicative Chernoff bound,
$\Pr\left(S_{N_{k}}^{(v)}\leq\frac{1}{2}N_{k}f\left(\frac{l_{k}}{4}\right)\mid E_{k,p,q}\right)\leq\exp\left(-N_{k}f\left(\frac{l_{k}}{4}\right)/8\right)$.

Now,

\begin{align*}
p_{1}\leq & \Pr\left(\text{CheckSignificant}\left(\left\{ B_{i}^{(v)}\right\} _{i=1}^{S_{N_{k}}^{(v)}},\frac{\delta}{4\log\frac{1}{2\epsilon}}\right)\text{ returns false},S_{N_{k}}^{(v)}\geq\frac{1}{2}N_{k}f\left(\frac{l_{k}}{4}\right)\mid E_{k,p,q}\right)\\
 & +\Pr\left(S_{N_{k}}^{(v)}<\frac{1}{2}N_{k}f\left(\frac{l_{k}}{4}\right)\mid E_{k,p,q}\right)
\end{align*}

By Condition~\ref{cond:oracle-noshape} and $\left|V_{k}-\theta^{*}\right|\geq\frac{l_{k}}{4}$,
$\mathbb{E}\left[B_{i}^{(v)}\mid E_{k,p,q}\right]\geq C\left(\frac{l_{k}}{4}\right)^{\beta}$.
On event $E_{k,p,q}$, $\left\{ B_{i}^{(v)}\right\} $ is a sequence
of i.i.d.\ random variables. Thus, On event $E_{k,p,q}$, by Lemma~\ref{lem:CheckSignificant-effectiveness},
with probability at least $1-\frac{\delta}{4\log\frac{1}{2\epsilon}}$,
CheckSignificant will return true if $\frac{1}{2}N_{k}f\left(\frac{l_{k}}{4}\right)=\Theta\left(\frac{1}{l_{k}^{2\beta}}\ln\frac{\ln1/\epsilon}{\delta}[\ln\ln]_{+}\frac{1}{l_{k}^{2\beta}}\right)$.
We have already proved $\Pr\left(S_{N_{k}}^{(v)}\leq\frac{1}{2}N_{k}f\left(\frac{l_{k}}{4}\right)\mid E_{k,p,q}\right)\leq\exp\left(-N_{k}f\left(\frac{l_{k}}{4}\right)/8\right)$.
By setting $N_{k}=\Theta\left(\frac{1}{f(l_{k}/4)}l_{k}^{-2\beta}\ln\frac{\ln1/\epsilon}{\delta}[\ln\ln]_{+}\frac{1}{l_{k}^{2\beta}}\right)$,
we can ensure $p_{1}$ is at most $\delta/2\log\frac{1}{2\epsilon}$.

Now we have proved on event $E_{k,p,q}$, if $\theta^{*}\in[L_{k},M_{k}]$,
then 
\[
\Pr\left(\text{CheckSignificant}\left(\left\{ B_{i}^{(v)}\right\} _{i=1}^{S_{N_{k}}^{(v)}},\frac{\delta}{4\log\frac{1}{2\epsilon}}\right)\text{ returns true}\mid E_{k,p,q}\right)\geq1-\frac{\delta}{2\log\frac{1}{2\epsilon}}
\]

Likewise, on event $E_{k,p,q}$, if $\theta^{*}\in[M_{k},R_{k}]$,
then 
\[
\Pr\left(\text{CheckSignificant}\left(\left\{ -B_{i}^{(u)}\right\} _{i=1}^{S_{N_{k}}^{(u)}},\frac{\delta}{4\log\frac{1}{2\epsilon}}\right)\text{ returns true}\mid E_{k,p,q}\right)\geq1-\frac{\delta}{2\log\frac{1}{2\epsilon}}
\]

Therefore, we have shown $\Pr\left(T_{k}\geq N_{k}\mid E_{k,p,q}\right)\leq\frac{\delta}{2\log\frac{1}{2\epsilon}}$
for any $k,p,q$. By (\ref{eq:label-decompose}), with probability
at least $1-\delta$, the number of samples queried is at most 
\begin{align*}
 & \sum_{k=0}^{\log\frac{1}{2\epsilon}-1}O\left(\frac{1}{f(\left(\frac{3}{4}\right)^{k}/4)}\left(\frac{3}{4}\right)^{-2\beta k}\ln\frac{\ln1/\epsilon}{\delta}[\ln\ln]_{+}\left(\frac{3}{4}\right)^{-2k\beta}\right)\\
= & O\left(\frac{\epsilon^{-2\beta}}{f(\epsilon/2)}\ln\frac{1}{\epsilon}\left(\ln\frac{1}{\delta}+\ln\ln\frac{1}{\epsilon}\right)[\ln\ln]_{+}\frac{1}{\epsilon}\right)
\end{align*}
\end{proof}

\begin{proof}[Proof of Theorem \ref{thm:algo-label-shape}]
For each $k$ in Algorithm~\ref{alg:The-active-learning} at Line
\ref{line:algo-1d:iter}, Let $l_{k}=R_{k}-L_{k}$. Let $N_{k}=\eta\frac{1}{f(l_{k}/4)}\ln\frac{4\log\frac{1}{2\epsilon}}{\delta}$,
where $\eta$ is a constant to be specified later. As with the previous
proof, it suffices to show $\Pr\left(T_{k}\geq N_{k}\mid E_{k,p,q}\right)\leq\frac{\delta}{2\log\frac{1}{2\epsilon}}$
where event $E_{k,p,q}$ is defined to be $L_{k}=p,R_{k}=q$, $T_{k}$
is the number of iterations at the loop at Line \ref{line:algo-1d:loop-n}.

On event $E_{k,p,q}$, we will show that the loop at Line \ref{line:algo-1d:loop-n}
will terminate after $n=N_{k}$ with probability at least $1-\frac{\delta}{2\log\frac{1}{2\epsilon}}$. 

Suppose for now $\theta^{*}\in[M_{k},R_{k}]$. Let $Z_{i}=A_{i}^{(u)}-A_{i}^{(m)}$,
$\zeta=\theta^{*}-M_{k}$. Clearly, $\left|Z_{i}\right|\leq1$. On
event $E_{k,p,q}$, sequence $\left\{ Z_{i}\right\} $ is i.i.d..
By Condition \ref{cond:oracle-shape}, $\mathbb{E}\left[Z_{i}\mid E_{k,p,q}\right]=f(\zeta+\frac{l_{k}}{4})-f(\zeta)\geq cf(\zeta+\frac{l_{k}}{4})$
since $\zeta\leq\frac{2}{3}(\zeta+\frac{l_{k}}{4})$. $\text{Var}\left[Z_{i}|E_{k,p,q}\right]=\text{Var}\left[A_{i}^{(u)}\mid E_{k,p,q}\right]+\text{Var}\left[A_{i}^{(m)}\mid E_{k,p,q}\right]\overset{(a)}{\leq}\mathbb{E}\left[A_{i}^{(u)}\mid E_{k,p,q}\right]+\mathbb{E}\left[A_{i}^{(m)}\mid E_{k,p,q}\right]=f(\zeta+\frac{l_{k}}{4})+f(\zeta)\overset{(b)}{\leq}2f(\zeta+\frac{l_{k}}{4})$
where (a) follows by $A_{i}\in\{0,1\}$ and (b) follows by the monotonicity
of $f$ . Thus, on event $E_{k,p,q}$, by Lemma~\ref{lem:CheckSignificant-var-effectiveness},
if we set $\eta$ sufficiently large (independent of $l_{k},\epsilon,\delta$),
then with probability at least $1-\frac{\delta}{4\log\frac{1}{2\epsilon}}$
CheckSignificant-Var$\left(\left\{ Z_{i}\right\} _{i=1}^{N_{k}},\frac{\delta}{4\log\frac{1}{2\epsilon}}\right)$
in Procedure~\ref{alg:test} returns true.

Similarly, we can show that on event $E_{k,p,q}$, if $\theta^{*}\in[L_{k},M_{k}]$,
by Lemma~\ref{lem:CheckSignificant-var-effectiveness}, with probability
at least $1-\frac{\delta}{4\log\frac{1}{2\epsilon}}$, CheckSignificant-Var$\left(\left\{ A_{i}^{(v)}-A_{i}^{(m)}\right\} _{i=1}^{N_{k}},\frac{\delta}{4\log\frac{1}{2\epsilon}}\right)$
returns true. 

Therefore, the loop at Line \ref{line:algo-1d:loop-n} will terminate
after $n=N_{k}$ with probability at least $1-\frac{\delta}{4\log\frac{1}{2\epsilon}}$
on event $E_{k,p,q}$. Therefore, with probability at least $1-\delta$,
the number of samples queried is at most $\sum_{k=0}^{\log\frac{1}{2\epsilon}-1}\frac{1}{f(\left(\frac{3}{4}\right)^{k}/4)}\ln\frac{\ln1/\epsilon}{\delta}=O\left(\frac{1}{f(\epsilon/2)}\ln\frac{1}{\epsilon}\left(\ln\frac{1}{\delta}+\ln\ln\frac{1}{\epsilon}\right)\right)$.
\end{proof}

\subsection{The d-dimensional case}

To prove the $d$-dimensional case, we only need to use a union bound
to show that with high probability all calls of Algorithm~\ref{alg:The-active-learning}
succeed, and consequently the output boundary $g$ produced by polynomial
interpolation is close to the true underlying boundary due to the
smoothness assumption of $g^{*}$.
\begin{proof}[Proof of Theorem~\ref{thm:algo-correctness-hd}]
For $q\in\left\{ 0,1,\dots,\frac{M}{\gamma}-1\right\} ^{d-1}$, define
the ``polynomial interpolation'' version of $g^{*}$ as 
\[
g_{q}^{*}(\tilde{\boldsymbol{x}})=\sum_{l\in I_{q}\cap\mathcal{L}}g^{*}(l)Q_{q,l}(\tilde{\boldsymbol{x}})
\]

Recall that we choose $M=O\left(\epsilon^{-1/\gamma}\right)$.

By Theorem~\ref{thm:algo-correctness}, each run of Algorithm~\ref{alg:The-active-learning}
at the line~\ref{line:hd:1d} of Algorithm~\ref{alg:The-active-learning-hd}
will return a $g_{l}$ such that $\left|g_{l}-g_{q}^{*}(l)\right|\leq\epsilon$
with probability at least $1-\delta/2M^{d-1}$.

\begin{align*}
 & \left\Vert g-g^{*}\right\Vert \\
= & \sum_{q\in\{0,\dots,M/\gamma-1\}^{d-1}}\left\Vert \left(g_{q}-g^{*}\right)\mathds{1}\{\boldsymbol{\tilde{x}}\in I_{q}\}\right\Vert \\
\leq & \sum_{q\in\{0,\dots,M/\gamma-1\}^{d-1}}\left\Vert \left(g_{q}-g_{q}^{*}\right)\mathds{1}\{\boldsymbol{\tilde{x}}\in I_{q}\}\right\Vert +\left\Vert \left(g_{q}^{*}-g^{*}\right)\mathds{1}\{\boldsymbol{\tilde{x}}\in I_{q}\}\right\Vert 
\end{align*}

\begin{eqnarray*}
\left\Vert \left(g_{q}^{*}-g^{*}\right)\mathds{1}\{\boldsymbol{\tilde{x}}\in I_{q}\}\right\Vert  & = & \int_{I_{q}}\left|g_{q}^{*}(\boldsymbol{\tilde{x}})-g^{*}(\boldsymbol{\tilde{x}})\right|d\boldsymbol{\tilde{x}}\\
 & = & O\left(\int_{I_{q}}M^{-\gamma}d\boldsymbol{\tilde{x}}\right)\\
 & = & O\left(M^{-\gamma-d+1}\right)
\end{eqnarray*}

The second equality follows from Lemma 3 of \cite{CN08} that $\left|g_{q}(\boldsymbol{\tilde{x}})-g^{*}(\boldsymbol{\tilde{x}})\right|=O\left(M^{-\gamma}\right)$
since $g^{*}$ is $\gamma$-Hölder smooth.

\begin{align*}
 & \left\Vert \left(g_{q}-g_{q}^{*}\right)\mathds{1}\{\boldsymbol{\tilde{x}}\in I_{q}\}\right\Vert \\
= & \sum_{l\in I_{q}\cap\mathcal{L}}\left|g_{l}-g_{q}^{*}(l)\right|\left\Vert Q_{q,l}\right\Vert \\
\leq & \sum_{l\in I_{q}\cap\mathcal{L}}\epsilon\left\Vert Q_{q}\right\Vert \\
= & O(\epsilon M^{-d+1})
\end{align*}

Therefore, overall we have $\left\Vert g-g^{*}\right\Vert \leq O\left(M^{-\gamma-d+1}+\epsilon M^{-d+1}\right)\left(\frac{M}{\gamma}\right)^{d-1}=O(\epsilon)$.
\end{proof}

\begin{proof}[Proof of Theorem~\ref{thm:algo-label-noshape--hd}]

By Theorem~\ref{thm:algo-label-noshape}, each run of Algorithm~\ref{alg:The-active-learning}
at the line~\ref{line:hd:1d} of Algorithm~\ref{alg:The-active-learning-hd}
will make $\tilde{O}\left(\frac{d}{f(\epsilon/2)}\epsilon^{-2\beta}\right)$
queries with probability at least $1-\delta/M^{d-1}$, thus by a union
bound, the total number of queries made is $\tilde{O}\left(\frac{d}{f(\epsilon/2)}\epsilon^{-2\beta-\frac{d-1}{\gamma}}\right)$
with probability at least $1-\delta$.
\end{proof}

\begin{proof}[Proof of Theorem~\ref{thm:algo-label-shape-hd}]

The proof is similar to the previous proof.
\end{proof}

\section{Proof of lower bounds}

First, we introduce some notations for this section. Given a labeler
$L$ and an active learning algorithm $\mathcal{A}$, denote by $P_{L,\mathcal{A}}^{n}$
the distribution of $n$ samples $\left\{ (X_{i},Y_{i})\right\} _{i=1}^{n}$
where $Y_{i}$ is drawn from distribution $P_{L}(Y|X_{i})$ and $X_{i}$
is drawn by the active learning algorithm based solely on the knowledge
of $\left\{ (X_{j},Y_{j})\right\} _{j=1}^{i-1}$. We will drop the
subscripts from $P_{L,\mathcal{A}}^{n}$ and $P_{\mathcal{L}}(Y|X)$
when it is clear from the context. For a sequence $\{X_{i}\}_{i=1}^{\infty}$
denote by $X^{n}$ the subsequence $\{X_{1},\dots,X_{n}\}$. 
\begin{defn}
For any distributions $P,Q$ on a countable support, define KL-divergence
as $d_{\text{KL}}\left(P,Q\right)=\sum\limits _{x}P(x)\ln\frac{P(x)}{Q(x)}$.
For two random variables $X,Y$, define the mutual information as
$I(X;Y)=d_{\text{KL}}\left(P(X,Y)\parallel P(X)P(Y)\right)$.

We will use Fano's method shown as below to prove the lower bounds.
\end{defn}
\begin{lem}
\label{lem:Fano}Let $\Theta$ be a class of parameters, and $\{P_{\theta}:\theta\in\Theta\}$
be a class of probability distributions indexed by $\Theta$ over
some sample space $\mathcal{X}$ . Let $d:\Theta\times\Theta\rightarrow\mathcal{\mathbb{R}}$
be a semi-metric. Let $\mathcal{V}=\left\{ \theta_{1},\dots,\theta_{M}\right\} \subseteq\Theta$
such that $\forall i\neq j$, $d(\theta_{i},\theta_{j})\geq2s>0$.
Let $\bar{P}=\frac{1}{M}\sum_{\theta\in\mathcal{V}}P_{\theta}$. If
$d_{\text{KL}}\left(P_{\theta}\parallel\bar{P}\right)\leq\delta$
for any $\theta\in\mathcal{V}$, then for any algorithm $\hat{\theta}$
that given a sample $X$ drawn from $P_{\theta}$ outputs $\hat{\theta}(X)\in\Theta$,
the following inequality holds:

\[
\sup_{\theta\in\Theta}P_{\theta}\left(d(\theta,\hat{\theta}(X))\geq s\right)\geq1-\frac{\delta+\ln2}{\ln M}
\]
\end{lem}
\begin{proof}
For any algorithm $\hat{\theta}$, define a test function $\hat{\Psi}:\mathcal{X}\rightarrow\{1,\dots,M\}$
such that $\hat{\Psi}(X)=\arg\min_{i\in\{1,\dots,M\}}d(\hat{\theta}(X),\theta_{i})$.
We have 

\[
\sup_{\theta\in\Theta}P_{\theta}\left(d(\theta,\hat{\theta}(X))\geq s\right)\geq\max_{\theta\in\mathcal{V}}P_{\theta}\left(d(\theta,\hat{\theta}(X))\geq s\right)\geq\max_{i\in\{1,\dots,M\}}P_{\theta_{i}}\left(\hat{\Psi}(X)\neq i\right)
\]

Let $V$ be a random variable uniformly taking values from $\mathcal{V}$,
and $X$ be drawn from $P_{V}$. By Fano's Inequality, for any test
function $\Psi:\mathcal{X}\rightarrow\{1,\dots,M\}$ 
\[
\max_{i\in\{1,\dots,M\}}P_{\theta_{i}}\left(\Psi(X)\neq i\right)\geq1-\frac{I(V;X)+\ln2}{\ln M}
\]

The desired result follows by the fact that $I(V;X)=\frac{1}{M}\sum_{\theta\in\mathcal{V}}d_{\text{KL}}\left(P_{\theta}\parallel\bar{P}\right)$.
\end{proof}

\subsection{The one dimensional case}
\begin{proof}[Proof of Theorem~\ref{thm:1d-lb-shape}]
\footnote{Actually we can use Le Cam's method to prove this one dimensional
case (which only needs to construct 2 distributions instead of 4 here),
but this proof can be generalized to the multidimensional case more
easily.} Without lose of generality, let $C=C'=1$ ($C$ is defined in Condition~\ref{cond:oracle-noshape}).
Let $\epsilon\leq\frac{1}{4}\min\left\{ \left(\frac{1}{2}\right)^{1/\beta},\left(\frac{4}{5}\right)^{1/\alpha},\frac{1}{4}\right\} $.
We will prove the desired result using Lemma~\ref{lem:Fano}. 

First, we construct $\mathcal{V}$ and $P_{\theta}$. For any $k\in\left\{ 0,1,2,3\right\} $,
let $P_{L_{k}}(Y\mid X)$ be the distribution of the labeler $L_{k}$'s
response with the ground truth $\theta_{k}=k\epsilon$: 

\begin{eqnarray*}
P_{L_{k}}\left(Y=\perp|x\right) & = & 1-\left|x-\frac{1}{2}-k\epsilon\right|^{\alpha}\\
P_{L_{k}}\left(Y=0|x\right) & = & \begin{cases}
\left(x-\frac{1}{2}-k\epsilon\right)^{\alpha}\left(1-\left(x-\frac{1}{2}-k\epsilon\right)^{\beta}\right)/2 & x>\frac{1}{2}+k\epsilon\\
\left(\frac{1}{2}+k\epsilon-x\right)^{\alpha}\left(1+\left(\frac{1}{2}+k\epsilon-x\right)^{\beta}\right)/2 & x\leq\frac{1}{2}+k\epsilon
\end{cases}\\
P_{L_{k}}\left(Y=1|x\right) & = & \begin{cases}
\left(x-\frac{1}{2}-k\epsilon\right)^{\alpha}\left(1+\left(x-\frac{1}{2}-k\epsilon\right)^{\beta}\right)/2 & x>\frac{1}{2}+k\epsilon\\
\left(\frac{1}{2}+k\epsilon-x\right)^{\alpha}\left(1-\left(\frac{1}{2}+k\epsilon-x\right)^{\beta}\right)/2 & x\leq\frac{1}{2}+k\epsilon
\end{cases}
\end{eqnarray*}

Clearly, $P_{L_{k}}$ complies with Conditions~\ref{cond:oracle-correct},
\ref{cond:oracle-noshape} and \ref{cond:oracle-shape}. 

Define $P_{k}^{n}$ to be the distribution of $n$ samples $\left\{ (X_{i},Y_{i})\right\} _{i=1}^{n}$
where $Y_{i}$ is drawn from distribution $P_{L_{k}}(Y|X_{i})$ and
$X_{i}$ is drawn by the active learning algorithm based solely on
the knowledge of $\left\{ (X_{j},Y_{j})\right\} _{j=1}^{i-1}$.

Define $\bar{P}_{L}=\frac{1}{4}\sum_{j}P_{L_{j}}$ and $\bar{P}^{n}=\frac{1}{4}\sum_{j}P_{k}^{n}$.
We take $\Theta$ to be $[0,1]$, and $d(\theta_{1},\theta_{2})=|\theta_{1}-\theta_{2}|$
in Lemma~\ref{lem:Fano}. To use Lemma~\ref{lem:Fano}, we need
to bound $d_{\text{KL}}\left(P_{k}^{n}\parallel\bar{P}^{n}\right)$
for $k\in\{0,1,2,3\}$.

For any $k\in\{0,1,2,3\}$ ,

\begin{align}
 & \text{\ensuremath{d_{\text{KL}}}}\left(P_{k}^{n}\parallel\bar{P}_{0}^{n}\right)\nonumber \\
= & \mathbb{E}_{P_{k}^{n}}\left(\ln\frac{P_{k}^{n}\left(\left\{ (X_{i},Y_{i})\right\} _{i=1}^{n}\right)}{\bar{P}^{n}\left(\left\{ (X_{i},Y_{i})\right\} _{i=1}^{n}\right)}\right)\nonumber \\
= & \mathbb{E}_{P_{k}^{n}}\left(\ln\frac{P_{k}^{n}\left(X_{1}\right)P_{k}^{n}\left(Y_{1}\mid X_{1}\right)P_{k}^{n}\left(X_{2}\mid X_{1},Y_{1}\right)\cdots P_{k}^{n}\left(Y_{n}\mid X_{1},Y_{1},\dots,X_{n}\right)}{\bar{P}^{n}\left(X_{1}\right)\bar{P}^{n}\left(Y_{1}\mid X_{1}\right)\bar{P}^{n}\left(X_{2}\mid X_{1},Y_{1}\right)\cdots\bar{P}^{n}\left(Y_{n}\mid X_{1},Y_{1},\dots,X_{n}\right)}\right)\nonumber \\
\stackrel{\text{(a)}}{=} & \mathbb{E}_{P_{k}^{n}}\left(\ln\frac{\Pi_{i=1}^{n}P_{L_{k}}\left(Y_{i}|X_{i}\right)}{\Pi_{i=1}^{n}\bar{P}_{L}\left(Y_{i}|X_{i}\right)}\right)\label{eq:kl-decompose}\\
= & \sum_{i=1}^{n}\mathbb{E}_{P_{k}^{n}}\left(\mathbb{E}_{P_{k}^{n}}\left(\ln\frac{P_{L_{k}}\left(Y_{i}|X_{i}\right)}{\bar{P}_{L}\left(Y_{i}|X_{i}\right)}\mid X^{n}\right)\right)\nonumber \\
\leq & n\max_{x\in[0,1]}\text{\ensuremath{d_{\text{KL}}}}\left(P_{L_{k}}(Y\mid x)\parallel\bar{P}_{L}(Y\mid x)\right)\nonumber 
\end{align}

(a) follows by the fact that $P_{k}^{n}\left(X_{i+1}\mid X_{1},Y_{1},\dots X_{i},Y_{i}\right)=\bar{P}^{n}\left(X_{i+1}\mid X_{1},Y_{1},\dots,X_{i},Y_{i}\right)$
since $X_{i+1}$ is drawn by the same active learning algorithm based
solely on the knowledge of $\left\{ (X_{j},Y_{j})\right\} _{j=1}^{i}$
regardless of the labeler's response distribution, and the fact that
$P_{k}^{n}\left(Y_{i}\mid X_{1},Y_{1},\dots,X_{i}\right)=P_{\mathcal{L}_{k}}\left(Y_{i}|X_{i}\right)$
and $\bar{P}^{n}\left(Y_{i}\mid X_{1},Y_{1},\dots,X_{i}\right)=\bar{P}_{L}\left(Y_{i}|X_{i}\right)$
by definition.

For any $k\in\{1,2,3\},x\in[0,1]$, 
\begin{equation}
\bar{P}_{L}(\cdot\mid x)\geq\frac{P_{L_{0}}(\cdot\mid x)+P_{L_{k}}(\cdot\mid x)}{4}\label{eq:1d-p-bar-lb}
\end{equation}

For any $k\in\{0,1,2,3\},x\in[0,1]$, $y\in\{1,-1,\perp\}$

\begin{align}
 & \left(\bar{P}_{L}(Y=y\mid x)-P_{L_{k}}(Y=y\mid x)\right)^{2}\nonumber \\
= & \left(\sum_{j}\frac{1}{4}\left(P_{L_{j}}(Y=y\mid x)-P_{L_{0}}(Y=y\mid x)\right)+\left(P_{L_{0}}(Y=y\mid\boldsymbol{x})-P_{L_{k}}(Y=y\mid x)\right)\right)^{2}\nonumber \\
\leq & \left(\frac{5}{16}\sum_{j>0}\left(P_{L_{j}}(Y=y\mid x)-P_{L_{0}}(Y=y\mid x)\right)^{2}+5\left(P_{L_{0}}(Y=y\mid x)-P_{L_{k}}(Y=y\mid x)\right)^{2}\right)\nonumber \\
\leq & 6\sum_{j>0}\left(P_{L_{j}}(Y=y\mid x)-P_{L_{0}}(Y=y\mid x)\right)^{2}\label{eq:1d-p-bar-diff-ub}
\end{align}

where the first inequality follows by $\left(\sum_{i=0}^{4}a_{i}\right)^{2}\leq5\sum_{i=0}^{4}a_{i}^{2}$
by letting $a_{j}=\frac{1}{4}\left(P_{L_{j}}(Y=y\mid x)-P_{L_{0}}(Y=y\mid x)\right)$
for $j=0,\dots,3$ and $a_{4}=P_{L_{0}}(Y=y\mid\boldsymbol{x})-P_{L_{k}}(Y=y\mid x)$,
and noting that $a_{0}=0$ under this setting.

Thus,

\begin{alignat*}{1}
 & \text{\ensuremath{d_{\text{KL}}}}\left(P_{L_{k}}(Y\mid x)\parallel\bar{P}_{L}(Y\mid x)\right)\\
\leq & \sum_{y}\frac{1}{\bar{P}_{L}(Y=y\mid\boldsymbol{x})}\left(P_{L_{k}}(Y=y\mid x)-\bar{P}_{L}(Y=y\mid x)\right)^{2}\\
\leq & 24\sum_{j>0}\sum_{y}\frac{1}{P_{L_{j}}(y\mid x)+P_{L_{0}}(y\mid x)}\left(P_{L_{j}}(Y=y\mid x)-P_{L_{0}}(Y=y\mid x)\right)^{2}\\
\leq & O(\epsilon^{\alpha})
\end{alignat*}

The first inequality follows from Lemma~\ref{lem:kl-l2}. The second
inequality follows by (\ref{eq:1d-p-bar-lb}) and (\ref{eq:1d-p-bar-diff-ub}).
The last inequality follows by applying Lemma~\ref{lem:kl-eps} to
$P_{L_{0}}(\cdot\mid x)$ and $P_{L_{j}}(\cdot\mid x)$ and the assumption
$\alpha\leq2$.

Therefore, we have $\text{\ensuremath{d_{\text{KL}}}}\left(P_{k}^{n}\parallel\bar{P}_{0}^{n}\right)=nO(\epsilon^{\alpha})$.
By setting $n=\epsilon^{-\alpha}$, we get $\text{\ensuremath{d_{\text{KL}}}}\left(P_{k}^{n}\parallel\bar{P}_{0}^{n}\right)\leq O\left(1\right)$,
and thus by Lemma~\ref{lem:Fano}, 
\[
\sup_{\theta}P_{\theta}\left(d(\theta,\hat{\theta}(X))\geq\Omega\left(\epsilon\right)\right)\ge1-\frac{O\left(1\right)+\ln2}{\ln4}=O\left(1\right)
\]
\end{proof}

\subsection{The d-dimensional case}

Again, we will use Lemma~\ref{lem:Fano} to prove the lower bounds
for $d$-dimensional cases. We first construct $\{P_{\theta}:\theta\in\Theta\}$
using a similar idea with \cite{CN08}, and then use Lemma~$\ref{lem:packing}$
to select a subset $\tilde{\Theta}\subset\Theta$ to apply Lemma~\ref{lem:Fano}.
\begin{proof}[Proof of Theorem~\ref{thm:hd-lb-noshape}]
 Again, without lose of generality, let $C=1$. Recall that for $\boldsymbol{x}=(x_{1},\dots,x_{d})\in\mathbb{R}^{d}$,
we have defined $\boldsymbol{\tilde{x}}$ to be $(x_{1},\dots,x_{d-1})$.
Define $m=\left(\frac{1}{\epsilon}\right)^{1/\gamma}$. $\mathcal{L}=\left\{ 0,\frac{1}{m},\dots,\frac{m-1}{m}\right\} ^{d-1}$,
$h(\tilde{\boldsymbol{x}})=\Pi_{i=1}^{d-1}\exp\left(-\frac{1}{1-4x_{i}^{2}}\right)\mathds{1}\left\{ |x_{i}|<\frac{1}{2}\right\} $,
$\phi_{l}(\tilde{\boldsymbol{x}})=Km^{-\gamma}h(m(\tilde{\boldsymbol{x}}-l)-\frac{1}{2})$
where $l\in\mathcal{L}$. It is easy to check $\phi_{l}(\tilde{\boldsymbol{x}})$
is $(K,\gamma)$-Hölder smooth and has bounded support $[l_{1},l_{1}+\frac{1}{m}]\times\cdots\times[l_{d-1},l_{d-1}+\frac{1}{m}]$,
which implies that for different $l_{1},l_{2}\in\mathcal{L}$, the
support of $\phi_{l_{1}}$ and $\phi_{l_{2}}$ do not intersect.

Let $\Omega=\{0,1\}^{m^{d-1}}$. For any $\boldsymbol{\omega}\in\Omega$,
define $g_{\omega}(\tilde{\boldsymbol{x}})=\sum_{l\in\mathcal{L}}\omega_{l}\phi_{l}(\tilde{\boldsymbol{x}})$.
For each $\boldsymbol{\omega}\in\Omega$, define the conditional distribution
of labeler $L_{\boldsymbol{\omega}}$'s response as follows: 

For $x_{d}\leq A$, $P_{L_{\boldsymbol{\omega}}}(y=\perp|\boldsymbol{x})=1-f(A)$,
$P_{L_{\boldsymbol{\omega}}}(y\neq\mathbb{I}(x_{d}>g_{\boldsymbol{\omega}}(\boldsymbol{\tilde{x}}))|\boldsymbol{x},y\neq\perp)=\frac{1}{2}\left(1-\left|x_{d}-g_{\boldsymbol{\omega}}(\boldsymbol{\tilde{x}})\right|^{\beta}\right)$;

For $x_{d}\geq A$, $P_{L_{\boldsymbol{\omega}}}(y=\perp|\boldsymbol{x})=1-f(x_{d})$,
$P_{L_{\boldsymbol{\omega}}}(y\neq\mathbb{I}(x_{d}>g_{\omega}(\boldsymbol{\tilde{x}}))|\boldsymbol{x},y\neq\perp)=\frac{1}{2}\left(1-x_{d}^{\beta}\right)$.

Here, $A=c\max\phi(\boldsymbol{\tilde{x}})=c'\epsilon$ for some constants
$c,c'$.

It can be easily verified that $P_{L_{\boldsymbol{\omega}}}$ satisfies
Conditions~\ref{cond:oracle-correct} and \ref{cond:oracle-noshape}.
Note that $g_{\boldsymbol{\omega}}(\boldsymbol{\tilde{x}})$ can be
seen as the underlying decision boundary for labeler $P_{L_{\boldsymbol{\omega}}}$.

Define $P_{\boldsymbol{\omega}}^{n}$ to be the distribution of $n$
samples $\left\{ (X_{i},Y_{i})\right\} _{i=1}^{n}$ where $Y_{i}$
is drawn from distribution $P_{L_{\boldsymbol{\omega}}}(Y|X_{i})$
and $X_{i}$ is drawn by the active learning algorithm based solely
on the knowledge of $\left\{ (X_{j},Y_{j})\right\} _{j=1}^{i-1}$.

By Lemma~\ref{lem:packing}, when $\epsilon$ is small enough so
that $m^{d-1}$ is large enough, there is a subset $\left\{ \boldsymbol{\omega}^{(1)},\dots,\boldsymbol{\omega}^{(M)}\right\} \subset\Omega$
such that $\left\Vert \boldsymbol{\omega}^{(i)}-\boldsymbol{\omega}^{(j)}\right\Vert _{0}\geq m^{d-1}/12$
for any $0\leq i<j\leq M$ and $M\geq2^{m^{d-1}/48}$. Define $P_{i}^{n}=P_{\boldsymbol{\omega}^{(i)}}^{n},\bar{P}^{n}=\frac{1}{M}\sum_{i=1}^{M}P_{i}^{n}$.

Next, we will apply Lemma~\ref{lem:Fano} to $\left\{ \boldsymbol{\omega}^{(1)},\dots,\boldsymbol{\omega}^{(M)}\right\} $
with $d(\boldsymbol{\omega}^{(i)},\boldsymbol{\omega}^{(j)})=\left\Vert g_{\boldsymbol{\omega}^{(i)}}-g_{\boldsymbol{\omega}^{(j)}}\right\Vert $.
We will lower-bound $d(\boldsymbol{\omega}^{(i)},\boldsymbol{\omega}^{(j)})$
and upper-bound $\text{\ensuremath{d_{\text{KL}}}}\left(P_{i}^{n}\parallel\bar{P^{n}}\right)$.

For any $1\leq i<j\leq M$ , 
\begin{align*}
 & \left\Vert g_{\boldsymbol{\omega}^{(i)}}-g_{\boldsymbol{\omega}^{(j)}}\right\Vert \\
= & \sum_{l\in\{1,\dots,m\}^{d-1}}\left|\omega_{l}^{(i)}-\omega_{l}^{(j)}\right|Km^{-\gamma-(d-1)}\left\Vert h\right\Vert \\
\geq & m^{d-1}/12*Km^{-\gamma-(d-1)}\left\Vert h\right\Vert \\
= & Km^{-\gamma}\left\Vert h\right\Vert /12\\
= & \Theta\left(\epsilon\right)
\end{align*}

By the convexity of KL-divergence, $\text{\ensuremath{d_{\text{KL}}}}\left(P_{i}^{n}\parallel\bar{P}^{n}\right)\leq\frac{1}{M}\sum_{j=1}^{M}\text{\ensuremath{d_{\text{KL}}}}\left(P_{i}^{n}\parallel P_{j}^{n}\right)$,
so it suffices to upper-bound $\text{\ensuremath{d_{\text{KL}}}}\left(P_{i}^{n}\parallel P_{j}^{n}\right)$
for any $i,j$.

For any $1<i,j\leq M$ , 
\begin{align*}
 & \text{\ensuremath{d_{\text{KL}}}}\left(P_{i}^{n}\parallel P_{j}^{n}\right)\\
\leq & n\max_{\boldsymbol{x}\in[0,1]^{d}}\text{\ensuremath{d_{\text{KL}}}}\left(P_{L_{\boldsymbol{\omega}^{(i)}}}^{n}(Y\mid\boldsymbol{x})\parallel P_{L_{\boldsymbol{\omega}^{(j)}}}^{n}(Y\mid\boldsymbol{x})\right)\\
= & n\max_{\boldsymbol{x}\in[0,1]^{d}}P_{L_{\boldsymbol{\omega}^{(i)}}}^{n}(Y\ne\perp\mid\boldsymbol{x})\text{\ensuremath{d_{\text{KL}}}}\left(P_{L_{\boldsymbol{\omega}^{(i)}}}^{n}(Y\mid\boldsymbol{x},Y\neq\perp)\parallel P_{L_{\boldsymbol{\omega}^{(j)}}}^{n}(Y\mid\boldsymbol{x},Y\neq\perp)\right)
\end{align*}

The inequality follows as (\ref{eq:kl-decompose}) in the proof of
Theorem~\ref{thm:1d-lb-shape}. The equality follows since $P_{\boldsymbol{\omega}}(y=\perp|\boldsymbol{x})$
is the same for all $\boldsymbol{\omega}\in\Omega$.

If $x_{d}\geq A$, then $P_{L_{\boldsymbol{\omega}^{(i)}}}^{n}(Y\mid\boldsymbol{x},Y\neq\perp)=P_{L_{\boldsymbol{\omega}^{(j)}}}^{n}(Y\mid\boldsymbol{x},Y\neq\perp)$,
so $\text{\ensuremath{d_{\text{KL}}}}\left(P_{L_{\boldsymbol{\omega}^{(i)}}}^{n}(Y\mid\boldsymbol{x},Y\neq\perp)\parallel P_{L_{\boldsymbol{\omega}^{(j)}}}^{n}(Y\mid\boldsymbol{x},Y\neq\perp)\right)=0$.
If $x_{d}<A$, then $P_{L_{\boldsymbol{\omega}^{(i)}}}^{n}(Y\ne\perp\mid\boldsymbol{x})=f(A)$.
Therefore, 
\[
\text{\ensuremath{d_{\text{KL}}}}\left(P_{i}^{n}\parallel P_{j}^{n}\right)\leq nf(A)\max_{\boldsymbol{x}\in[0,1]^{d}}d_{\text{KL}}\left(P_{L_{\boldsymbol{\omega}^{(i)}}}^{n}(Y\mid\boldsymbol{x},Y\neq\perp)\parallel P_{L_{\boldsymbol{\omega}^{(j)}}}^{n}(Y\mid\boldsymbol{x},Y\neq\perp)\right)
\]
.

Apply Lemma~\ref{lem:kl-l2} to $P_{L_{\boldsymbol{\omega}^{(i)}}}^{n}(Y\mid\boldsymbol{x},Y\neq\perp)$
and $P_{L_{\boldsymbol{\omega}^{(i)}}}^{n}(Y\mid\boldsymbol{x},Y\neq\perp)$,
and noting they are bounded above by a constant, we have $\max_{\boldsymbol{x}\in[0,1]^{d}}\text{\ensuremath{d_{\text{KL}}}}\left(P_{L_{\boldsymbol{\omega}^{(i)}}}^{n}(Y\mid\boldsymbol{x},Y\neq\perp)\parallel P_{L_{\boldsymbol{\omega}^{(j)}}}^{n}(Y\mid\boldsymbol{x},Y\neq\perp)\right)=O\left(A^{2\beta}\right)$.
Thus, 
\[
\text{\ensuremath{d_{\text{KL}}}}\left(P_{i}^{n}\parallel P_{j}^{n}\right)\leq nf(A)O\left(A^{2\beta}\right)=nf(c'\epsilon)O(\epsilon^{2\beta})
\]

By setting $n=\frac{1}{f(c'\epsilon)}\epsilon^{-2\beta-\frac{d-1}{\gamma}}$,
we get $\text{\ensuremath{d_{\text{KL}}}}\left(P_{i}^{n}\parallel P_{j}^{n}\right)\leq O\left(\epsilon^{-\frac{d-1}{\gamma}}\right)$.
The desired results follows by Lemma~\ref{lem:Fano}.
\end{proof}

The proof of Theorem~\ref{thm:hd-lb-shape} follows the same structure.
\begin{proof}[Proof of Theorem~\ref{thm:hd-lb-shape}]
 As in the proof of Theorem~\ref{thm:hd-lb-noshape}, let $C=C'=1$,
and define $m=\left(\frac{1}{\epsilon}\right)^{1/\gamma}$. $\mathcal{L}=\left\{ 0,\frac{1}{m},\dots,\frac{m-1}{m}\right\} ^{d-1}$,
$h(\tilde{\boldsymbol{x}})=\Pi_{i=1}^{d-1}\exp\left(-\frac{1}{1-4x_{i}^{2}}\right)\mathds{1}\left\{ |x_{i}|<\frac{1}{2}\right\} $,
$\phi_{l}(\tilde{\boldsymbol{x}})=Km^{-\gamma}h(m(\tilde{\boldsymbol{x}}-l)-\frac{1}{2})$
where $l\in\mathcal{L}$. Let $\Omega=\{0,1\}^{m^{d-1}}$. For any
$\boldsymbol{\omega}\in\Omega$, define $g_{\omega}(\tilde{\boldsymbol{x}})=\frac{1}{2}+\sum_{l\in\mathcal{L}}\omega_{l}\phi_{l}(\tilde{\boldsymbol{x}})$,
which can be seen as a decision boundary. $A=\max\phi(\boldsymbol{\tilde{x}})=c'\epsilon$
for some constants $c'$.

Let $g_{+}(\tilde{\boldsymbol{x}})=g_{(1,1,\dots,1)}(\tilde{\boldsymbol{x}})=\sum_{l\in\mathcal{L}}\phi_{l}(\tilde{\boldsymbol{x}})$,
$g_{-}(\tilde{\boldsymbol{x}})=g_{(0,0,\dots,0)}(\tilde{\boldsymbol{x}})=0$.
In other words, $g_{+}$ is the ``highest'' boundary, and $g_{-}$
is the ``lowest'' boundary.

For each $\boldsymbol{\omega}\in\Omega$, define the conditional distribution
of labeler $L_{\boldsymbol{\omega}}$'s response as follows: 

\[
P_{L_{\boldsymbol{\omega}}}(y=\perp|\boldsymbol{x})=1-\left|x_{d}-g_{\boldsymbol{\omega}}(\boldsymbol{\tilde{x}})\right|^{\alpha}
\]

\[
P_{L_{\boldsymbol{\omega}}}(y\neq\mathbb{I}(x_{d}>g_{\boldsymbol{\omega}}(\boldsymbol{\tilde{x}}))|\boldsymbol{x},y\neq\perp)=\frac{1}{2}\left(1-\left|x_{d}-g_{\boldsymbol{\omega}}(\boldsymbol{\tilde{x}})\right|^{\beta}\right)
\]

It can be easily verified that $P_{L_{\boldsymbol{\omega}}}$ satisfies
Conditions~\ref{cond:oracle-correct}, \ref{cond:oracle-noshape},
and \ref{cond:oracle-shape}.

Let $P_{+}(\cdot\mid\boldsymbol{x})=P_{L_{(1,1,\dots,1)}}(\cdot\mid\boldsymbol{x})$,
$P_{-}(\cdot\mid\boldsymbol{x})=P_{L_{(0,0,\dots,0)}}(\cdot\mid\boldsymbol{x})$.
By the construction of $g$, for any $\boldsymbol{x}\in[0,1]^{d}$,
any $\boldsymbol{\omega}\in\Omega$, $P_{L_{\boldsymbol{\omega}}}(\cdot\mid\boldsymbol{x})$
equals either $P_{+}(\cdot\mid\boldsymbol{x})$ or $P_{-}(\cdot\mid\boldsymbol{x})$. 

Define $P_{\boldsymbol{\omega}}^{n}$ to be the distribution of $n$
samples $\left\{ (X_{i},Y_{i})\right\} _{i=1}^{n}$ where $Y_{i}$
is drawn from distribution $P_{L_{\boldsymbol{\omega}}}(Y|X_{i})$
and $X_{i}$ is drawn by the active learning algorithm based solely
on the knowledge of $\left\{ (X_{j},Y_{j})\right\} _{j=1}^{i-1}$. 

By Lemma~\ref{lem:packing}, when $\epsilon$ is small enough so
that $m^{d-1}$ is large enough,, there is a subset $\Omega'=\left\{ \boldsymbol{\omega}^{(1)},\dots,\boldsymbol{\omega}^{(M)}\right\} \subset\Omega$
such that (i) (well-separated) $\left\Vert \boldsymbol{\omega}^{(i)}-\boldsymbol{\omega}^{(j)}\right\Vert _{0}\geq m^{d-1}/12$
for any $0\leq i<j\leq M$, $M\geq2^{m^{d-1}/48}$; and (ii) (well-balanced)
for any $j=1,\dots,m^{d-1}$, $\frac{1}{24}\leq\frac{1}{M}\sum_{i=1}^{M}\boldsymbol{\omega}_{j}^{(i)}\leq\frac{3}{24}$
. 

Define $P_{i}^{n}=P_{\boldsymbol{\omega}^{(i)}}^{n},\bar{P}^{n}=\frac{1}{M}\sum_{i=1}^{M}P_{i}^{n}$.
Define $P_{L_{i}}=P_{L_{\boldsymbol{\omega}^{(i)}}}$, $\bar{P}_{L}=\frac{1}{M}\sum_{i=1}^{M}P_{L_{i}}$.
By the well-balanced property, for any $\boldsymbol{x}\in[0,1]^{d}$,
$\bar{P}_{L}(\cdot\mid\boldsymbol{x})$ is between $\frac{1}{24}P_{+}(\cdot\mid\boldsymbol{x})+\frac{23}{24}P_{-}(\cdot\mid\boldsymbol{x})$
and $\frac{3}{24}P_{+}(\cdot\mid\boldsymbol{x})+\frac{21}{24}P_{-}(\cdot\mid\boldsymbol{x})$.
Therefore 
\begin{equation}
\bar{P}_{L}(\cdot\mid\boldsymbol{x})\geq\frac{1}{24}\left(P_{+}(\cdot\mid\boldsymbol{x})+P_{-}(\cdot\mid\boldsymbol{x})\right)\label{eq:hd-p-bar-lb}
\end{equation}

Moreover, since $P_{L_{i}}(\cdot\mid\boldsymbol{x})$ can only take
$P_{+}(\cdot\mid\boldsymbol{x})$ or $P_{-}(\cdot\mid\boldsymbol{x})$
for any $\boldsymbol{x}$, 
\begin{equation}
\left|P_{L_{i}}(\cdot\mid\boldsymbol{x})-\bar{P}_{\mathcal{L}}(\cdot\mid\boldsymbol{x})\right|\leq\left|P_{+}(\cdot\mid\boldsymbol{x})-P_{-}(\cdot\mid\boldsymbol{x})\right|\label{eq:hd-p-bar-diff-ub}
\end{equation}

Next, we will apply Lemma~\ref{lem:Fano} to $\left\{ \boldsymbol{\omega}^{(1)},\dots,\boldsymbol{\omega}^{(M)}\right\} $
with $d(\boldsymbol{\omega}^{(i)},\boldsymbol{\omega}^{(j)})=\left\Vert g_{\boldsymbol{\omega}^{(i)}}-g_{\boldsymbol{\omega}^{(j)}}\right\Vert $.
We already know from the proof of Theorem~\ref{thm:hd-lb-noshape}
$\left\Vert g_{\boldsymbol{\omega}^{(i)}}-g_{\boldsymbol{\omega}^{(j)}}\right\Vert =\Omega\left(\epsilon\right)$.

For any $0<i\leq M$ , $\text{\ensuremath{d_{\text{KL}}}}\left(P_{i}^{n}\parallel\bar{P}_{0}^{n}\right)\leq n\max_{\boldsymbol{x}\in[0,1]^{d}}\text{\ensuremath{d_{\text{KL}}}}\left(P_{L_{i}}(Y\mid\boldsymbol{x})\parallel\bar{P}_{L}(Y\mid\boldsymbol{x})\right)$.
For any $\boldsymbol{x}\in[0,1]^{d}$, 
\begin{alignat*}{1}
 & \text{\ensuremath{d_{\text{KL}}}}\left(P_{L_{i}}(Y\mid\boldsymbol{x})\parallel\bar{P}_{L}(Y\mid\boldsymbol{x})\right)\\
\leq & \sum_{y}\frac{1}{\bar{P}_{L}(Y=y\mid\boldsymbol{x})}\left(P_{L_{i}}(Y=y\mid\boldsymbol{x})-\bar{P}_{L}(Y=y\mid\boldsymbol{x})\right)^{2}\\
\leq & \sum_{y}\frac{24}{P_{+}(y\mid\boldsymbol{x})+P_{-}(y\mid\boldsymbol{x})}\left(P_{+}(Y=y\mid\boldsymbol{x})-P_{-}(Y=y\mid\boldsymbol{x})\right)^{2}\\
\leq & O(A^{\alpha})
\end{alignat*}

The first inequality follows from Lemma~\ref{lem:kl-l2}. The second
inequality follows by (\ref{eq:hd-p-bar-lb}) and (\ref{eq:hd-p-bar-diff-ub}).
The last inequality follows by applying Lemma~\ref{lem:kl-eps} to
$P_{+}(\cdot\mid\boldsymbol{x})$ and $P_{-}(\cdot\mid\boldsymbol{x})$,
setting the $\epsilon$ in Lemma~\ref{lem:kl-eps} to be $g_{\boldsymbol{\omega}}(\boldsymbol{\tilde{x}})$,
and using $g_{\boldsymbol{\omega}}(\boldsymbol{\tilde{x}})\leq A$
and the assumption $\alpha\leq2$.

Therefore, we have 

\[
\text{\ensuremath{d_{\text{KL}}}}\left(P_{i}^{n}\parallel P_{0}^{n}\right)\leq nO\left(A^{\alpha}\right)=nO(\epsilon^{\alpha})
\]

By setting $n=\epsilon^{-\alpha-\frac{d-1}{\gamma}}$, we get $\text{\ensuremath{d_{\text{KL}}}}\left(P_{i}^{n}\parallel P_{0}^{n}\right)\leq O\left(\epsilon^{-\frac{d-1}{\gamma}}\right)$
. Thus by Lemma \ref{lem:Fano}, 
\[
\sup_{\theta}P_{\theta}\left(d(\theta,\hat{\theta}(X))\geq\Omega\left(\epsilon\right)\right)\ge1-\frac{O\left(\epsilon^{-\frac{d-1}{\gamma}}\right)+\ln2}{\epsilon^{-\frac{d-1}{\gamma}}/48}=O\left(1\right)
\]
, from which the desired result follows.
\end{proof}

\section{Technical lemmas}

\subsection{Concentration bounds}

In this subsection, we define $Y_{1},Y_{2},\dots$ to be a sequence
of i.i.d.\ random variables. Assume $Y_{1}\in[-2,2]$, $\mathbb{E}Y_{1}=0$,
$\text{Var}(Y_{1})=\sigma^{2}\leq4$. Define $V_{n}=\frac{n}{n-1}\left(\sum_{i=1}^{n}Y_{i}^{2}-\frac{1}{n}\left(\sum_{i=1}^{n}Y_{i}\right)^{2}\right)$.
It is easy to check $\mathbb{E}V_{n}=n\sigma^{2}$.

We need following two results from \cite{ramdas2015sequential}
\begin{lem}
\label{lem:uniform-bernstein}(\cite{ramdas2015sequential}, Theorem
2) Take any $0<\delta<1$. Then there is an absolute constant $D_{0}$
such that with probability at least $1-\delta$, for all $n$ simultaneously, 

\[
\left|\sum_{i=1}^{n}Y_{i}\right|\leq D_{0}\left(1+\ln\frac{1}{\delta}+\sqrt{n\sigma^{2}\left[\ln\ln\right]_{+}(n\sigma^{2})+n\sigma^{2}\ln\frac{1}{\delta}}\right)
\]
\end{lem}

\begin{lem}
\label{lem:var-sumsqr}(\cite{ramdas2015sequential}, Lemma 3) Take
any $0<\delta<1$. Then there is an absolute constant $K_{0}$ such
that with probability at least $1-\delta$, for all $n$ simultaneously, 

\[
n\sigma^{2}\leq K_{0}\left(1+\ln\frac{1}{\delta}+\sum_{i=1}^{n}Y_{i}^{2}\right)
\]
\end{lem}
We note that Proposition~\ref{prop:uniform-berstein} is immediate
from Lemma~\ref{lem:uniform-bernstein} since $\text{Var}(Y_{i})\leq4$.

\begin{lem}
\label{lem:var-leq-emp}Take any $0<\delta<1$. Then there is an absolute
constant $K_{3}$ such that with probability at least $1-\delta$,
for all $n\geq\ln\frac{1}{\delta}$ simultaneously, 

\[
n\sigma^{2}\leq K_{3}\left(1+\ln\frac{1}{\delta}+V_{n}\right)
\]
\end{lem}
\begin{proof}
By Lemma~\ref{lem:var-sumsqr}, with probability at least $1-\delta/2$,
for all $n$,
\[
n\sigma^{2}\leq K_{0}\left(\sum_{i=1}^{n}Y_{i}^{2}+\ln\frac{2}{\delta}+1\right)=K_{0}\left(\frac{n-1}{n}V_{n}+\frac{1}{n}\left(\sum_{i=1}^{n}Y_{i}\right)^{2}+\ln\frac{2}{\delta}+1\right)
\]

By Lemma ~\ref{lem:uniform-bernstein}, with probability at least
$1-\delta/2$, for all $n$,
\begin{eqnarray*}
\frac{1}{n}\left(\sum_{i=1}^{n}Y_{i}\right)^{2} & < & \frac{1}{n}\left(D_{0}\left(1+\ln\frac{2}{\delta}+\sqrt{n\sigma^{2}\left[\ln\ln\right]_{+}(n\sigma^{2})+n\sigma^{2}\ln\frac{2}{\delta}}\right)\right)^{2}\\
 & = & \frac{D_{0}^{2}}{n}\left(1+\ln\frac{2}{\delta}\right)^{2}+D_{0}^{2}\sigma^{2}\left[\ln\ln\right]_{+}(n\sigma^{2})+D_{0}^{2}\sigma^{2}\ln\frac{2}{\delta}\\
 &  & +2D_{0}^{2}\left(1+\ln\frac{2}{\delta}\right)\sqrt{\frac{\sigma^{2}\left[\ln\ln\right]_{+}(n\sigma^{2})+\sigma^{2}\ln\frac{2}{\delta}}{n}}\\
 & \leq & K_{1}\left(1+\ln\frac{1}{\delta}+\left[\ln\ln\right]_{+}(n\sigma^{2})\right)
\end{eqnarray*}

for some absolute constant $K_{1}$. The last inequality follows by
$n\geq\ln\frac{1}{\delta}$.

Thus, by a union bound, with probability at least $1-\delta$, for
all $n$, $n\sigma^{2}\leq K_{0}V_{n}+K_{0}(K_{1}+2)\ln\frac{1}{\delta}+K_{0}K_{1}\left[\ln\ln\right]_{+}(n\sigma^{2})+K_{0}(K_{1}+3)$. 

Let $K_{2}>0$ be an absolute constant such that $\forall x\geq K_{2}$,
$K_{0}K_{1}\left[\ln\ln\right]_{+}x\leq\frac{x}{2}$.

Now if $n\sigma^{2}\geq K_{2}$, then $n\sigma^{2}\leq K_{0}V_{n}+K_{0}(K_{1}+2)\ln\frac{1}{\delta}+\frac{n\sigma^{2}}{2}+K_{0}(K_{1}+3)$,
and thus 
\begin{equation}
n\sigma^{2}\leq2K_{0}V_{n}+2K_{0}(K_{1}+2)\ln\frac{1}{\delta}+2K_{0}(K_{1}+3)+K_{2}\label{eq:lem-var-leq-emp}
\end{equation}

If $n\sigma^{2}\leq K_{2}$, clearly (\ref{eq:lem-var-leq-emp}) holds.
This concludes the proof.
\end{proof}
We note that Proposition~\ref{prop:uniform-empirical-berstein} is
immediate by applying above lemma to Lemma~\ref{lem:uniform-bernstein}.
\begin{lem}
\label{lem:single-var-bound}Take any $\delta,n>0$. Then with probability
at least $1-\delta$,

\[
V_{n}\leq4n\sigma^{2}+8\ln\frac{1}{\delta}
\]
\end{lem}
\begin{proof}
Applying Bernstein's Inequality to $Y_{i}^{2},$ and noting that $\text{Var}(Y_{i}^{2})\leq4\sigma^{2}$
since $|Y_{i}|\leq2$, we have with probability at least $1-\delta$,
\begin{eqnarray*}
\sum_{i=1}^{n}Y_{i}^{2} & \leq & \frac{4}{3}\ln\frac{1}{\delta}+n\sigma^{2}+\sqrt{8n\sigma^{2}\ln\frac{1}{\delta}}\\
 & \leq & 4\ln\frac{1}{\delta}+2n\sigma^{2}
\end{eqnarray*}

The last inequality follows by the fact that $\sqrt{4ab}\leq a+b$.

The desired result follows by noting that $V_{n}=\frac{n}{n-1}\left(\sum_{i=1}^{n}Y_{i}^{2}-\frac{1}{n}\left(\sum_{i=1}^{n}Y_{i}\right)^{2}\right)\leq2\sum_{i=1}^{n}Y_{i}^{2}$.
\end{proof}

\subsection{Bounds of distances among probability distributions}
\begin{lem}
\label{lem:kl-l2}If $P,Q$ are two probability distributions on a
countable support $\mathcal{X}$, then 
\[
d_{\text{KL}}\left(P\parallel Q\right)\leq\sum_{x}\frac{\left(P(x)-Q(x)\right)^{2}}{Q(x)}
\]
\end{lem}
\begin{proof}
\begin{eqnarray*}
d_{\text{KL}}\left(P\parallel Q\right) & = & \sum_{x}P(x)\ln\frac{P(x)}{Q(x)}\\
 & \leq & \sum_{x}P(x)\left(\frac{P(x)}{Q(x)}-1\right)\\
 & = & \sum_{x}\frac{\left(P(x)-Q(x)\right)^{2}}{Q(x)}
\end{eqnarray*}

The first inequality follows by $\ln x\leq x-1$. The second equality
follows by $\sum_{x}P(x)\left(\frac{P(x)}{Q(x)}-1\right)=\sum_{x}\left(\frac{P^{2}(x)-P(x)Q(x)}{Q(x)}-P(x)+Q(x)\right)=\sum_{x}\frac{\left(P(x)-Q(x)\right)^{2}}{Q(x)}$.
\end{proof}
Define 
\begin{eqnarray*}
P_{0}\left(Y=\perp|x\right) & = & 1-\left|x-\frac{1}{2}\right|^{\alpha}\\
P_{0}\left(Y=0|x\right) & = & \begin{cases}
\left(x-\frac{1}{2}\right)^{\alpha}\left(1-\left(x-\frac{1}{2}\right)^{\beta}\right)/2 & x>\frac{1}{2}\\
\left(\frac{1}{2}-x\right)^{\alpha}\left(1+\left(\frac{1}{2}-x\right)^{\beta}\right)/2 & x\leq\frac{1}{2}
\end{cases}\\
P_{0}\left(Y=1|x\right) & = & \begin{cases}
\left(x-\frac{1}{2}\right)^{\alpha}\left(1+\left(x-\frac{1}{2}\right)^{\beta}\right)/2 & x>\frac{1}{2}\\
\left(\frac{1}{2}-x\right)^{\alpha}\left(1-\left(\frac{1}{2}-x\right)^{\beta}\right)/2 & x\leq\frac{1}{2}
\end{cases}
\end{eqnarray*}

and

\begin{eqnarray*}
P_{1}\left(Y=\perp|x\right) & = & 1-\left|x-\epsilon-\frac{1}{2}\right|^{\alpha}\\
P_{1}\left(Y=0|x\right) & = & \begin{cases}
\left(x-\epsilon-\frac{1}{2}\right)^{\alpha}\left(1-\left(x-\epsilon-\frac{1}{2}\right)^{\beta}\right)/2 & x>\epsilon+\frac{1}{2}\\
\left(\epsilon+\frac{1}{2}-x\right)^{\alpha}\left(1+\left(\epsilon+\frac{1}{2}-x\right)^{\beta}\right)/2 & x\leq\epsilon+\frac{1}{2}
\end{cases}\\
P_{1}\left(Y=1|x\right) & = & \begin{cases}
\left(x-\epsilon-\frac{1}{2}\right)^{\alpha}\left(1+\left(x-\epsilon-\frac{1}{2}\right)^{\beta}\right)/2 & x>\epsilon+\frac{1}{2}\\
\left(\epsilon+\frac{1}{2}-x\right)^{\alpha}\left(1-\left(\epsilon+\frac{1}{2}-x\right)^{\beta}\right)/2 & x\leq\epsilon+\frac{1}{2}
\end{cases}
\end{eqnarray*}

\begin{lem}
\label{lem:kl-eps}Let $P_{0}$, $P_{1}$ be the distributions defined
above. If $x\in[0,1]$, $\epsilon\leq\min\left\{ \left(\frac{1}{2}\right)^{1/\beta},\left(\frac{4}{5}\right)^{1/\alpha},\frac{1}{4}\right\} $,
then 
\begin{equation}
\sum_{y}\frac{\left(P_{0}(Y=y|x)-P_{1}(Y=y|x)\right)^{2}}{P_{0}(Y=y|x)+P_{1}(Y=y|x)}=O\left(\epsilon^{\alpha}+\epsilon^{2}\right)\label{eq:lemma24-main}
\end{equation}
\end{lem}
\begin{proof}
By symmetry, it suffices to show for $0\leq x\leq\frac{1+\epsilon}{2}$.
Let $t=\frac{1}{2}+\epsilon-x$.

We first show (\ref{eq:lemma24-main}) holds for $\frac{\epsilon}{2}\leq t\leq\epsilon$
(i.e. $\frac{1}{2}\leq x\leq\frac{1+\epsilon}{2}$). 

We claim $\min_{y}\left(P_{0}(Y=y|X=t)+P_{1}(Y=y|X=t)\right)\geq\frac{1}{2}\left(\frac{\epsilon}{2}\right)^{\alpha}$.
This is because:

\begin{itemize}
\item $P_{0}(Y=\perp|X=t)+P_{1}(Y=\perp|X=t)=1-\left(\epsilon-t\right)^{\alpha}+1-t^{\alpha}\geq2-2\epsilon^{\alpha}\geq\frac{1}{2}\left(\frac{\epsilon}{2}\right)^{\alpha}$
where the last inequality follows by $\epsilon\leq\left(\frac{4}{5}\right)^{1/\alpha}$;
\item $2\left(P_{0}(Y=0|X=t)+P_{1}(Y=0|X=t)\right)=\left(\epsilon-t\right)^{\alpha}\left(1-\left(\epsilon-t\right)^{\beta}\right)+t^{\alpha}\left(1+t^{\beta}\right)\geq t^{\alpha}\left(1+t^{\beta}\right)\ge\left(\frac{\epsilon}{2}\right)^{\alpha}$.
Therefore, $P_{0}(Y=0|X=t)+P_{1}(Y=0|X=t)\geq\frac{1}{2}\left(\frac{\epsilon}{2}\right)^{\alpha}$. 
\item Similarly, $P_{0}(Y=1|X=t)+P_{1}(Y=1|X=t)\geq\frac{1}{2}\left(\frac{\epsilon}{2}\right)^{\alpha}$.
\end{itemize}
Besides, 
\begin{align*}
 & \sum_{y}\left(P_{0}(Y=y|X=t)-P_{1}(Y=y|X=t)\right)^{2}\\
= & \left(t^{\alpha}-\left(\epsilon-t\right)^{\alpha}\right)^{2}+\frac{1}{4}\left(t^{\alpha}\left(1-t^{\beta}\right)-\left(\epsilon-t\right)^{\alpha}\left(1+\left(\epsilon-t\right)^{\beta}\right)\right)^{2}\\
 & +\frac{1}{4}\left(t^{\alpha}\left(1+t^{\beta}\right)-\left(\epsilon-t\right)^{\alpha}\left(1-\left(\epsilon-t\right)^{\beta}\right)\right)^{2}\\
= & \left(t^{\alpha}-\left(\epsilon-t\right)^{\alpha}\right)^{2}+\frac{1}{4}\left(t^{\alpha}-\left(\epsilon-t\right)^{\alpha}-t^{\alpha+\beta}-\left(\epsilon-t\right)^{\alpha+\beta}\right)^{2}\\
 & +\frac{1}{4}\left(t^{\alpha}-\left(\epsilon-t\right)^{\alpha}+t^{\alpha+\beta}+\left(\epsilon-t\right)^{\alpha+\beta}\right)^{2}\\
\overset{(a)}{\leq} & \left(t^{\alpha}-\left(\epsilon-t\right)^{\alpha}\right)^{2}+\frac{1}{2}\left(t^{\alpha}-\left(\epsilon-t\right)^{\alpha}\right)^{2}+\frac{1}{2}\left(t^{\alpha+\beta}+\left(\epsilon-t\right)^{\alpha+\beta}\right)^{2}\\
 & +\frac{1}{2}\left(t^{\alpha}-\left(\epsilon-t\right)^{\alpha}\right)^{2}+\frac{1}{2}\left(t^{\alpha+\beta}+\left(\epsilon-t\right)^{\alpha+\beta}\right)^{2}\\
= & 2\left(t^{\alpha}-\left(\epsilon-t\right)^{\alpha}\right)^{2}+\left(t^{\alpha+\beta}+\left(\epsilon-t\right)^{\alpha+\beta}\right)^{2}\\
\leq & 2\epsilon^{2\alpha}+4\epsilon^{2\alpha+2\beta}\\
\leq & 6\epsilon^{2\alpha}
\end{align*}

where (a) follows by the inequality $(a+b)^{2}\leq2a^{2}+2b^{2}$
for any $a,b$.

Therefore, we get $\sum_{y}\frac{\left(P_{0}(Y=y|x)-P_{1}(Y=y|x)\right)^{2}}{P_{0}(Y=y|x)+P_{1}(Y=y|x)}\leq\frac{\sum_{y}\left(P_{0}(Y=y|x)-P_{1}(Y=y|x)\right)^{2}}{\min_{y}\left(P_{0}(Y=y|x)+P_{1}(Y=y|x)\right)}\leq12*2^{\alpha}\epsilon^{\alpha}$
when $\frac{1}{2}\leq x\leq\frac{1+\epsilon}{2}$.

Next, We show (\ref{eq:lemma24-main}) holds for $\epsilon\leq t\leq\frac{1}{2}+\epsilon$
(i.e. $0\leq x\leq\frac{1}{2}$). We will show $\frac{\left(P_{0}(Y=y|x)-P_{1}(Y=y|x)\right)^{2}}{P_{0}(Y=y|x)+P_{1}(Y=y|x)}=O\left(\epsilon^{\alpha}+\epsilon^{2}\right)$
for $Y=\perp,1,0$.

For $Y=\perp$, for the denominator, 
\[
P_{0}(Y=\perp|X=t)+P_{1}(Y=\perp|X=t)=2-t^{\alpha}-\left(t-\epsilon\right)^{\alpha}\geq2-\left(\frac{3}{4}\right)^{\alpha}-\left(\frac{1}{2}\right)^{\alpha}
\]
For the numerator, 
\[
\left(P_{0}(Y=\perp|X=t)-P_{1}(Y=\perp|X=t)\right)^{2}=\left(t^{\alpha}-\left(t-\epsilon\right)^{\alpha}\right)^{2}=t^{2\alpha}\left(1-\left(1-\frac{\epsilon}{t}\right)^{\alpha}\right)^{2}
\]
By Lemma~\ref{lem:bernoulli}, if $\alpha\geq1$, $t^{2\alpha}\left(1-\left(1-\frac{\epsilon}{t}\right)^{\alpha}\right)^{2}\leq t^{2\alpha}\left(\alpha\frac{\epsilon}{t}\right)^{2}=t^{2\alpha-2}\left(\alpha\epsilon\right)^{2}=O\left(\epsilon^{2}\right)$.
If $0\leq\alpha\leq1$, $t^{2\alpha}\left(1-\left(1-\frac{\epsilon}{t}\right)^{\alpha}\right)^{2}\leq t^{2\alpha}\left(\frac{\epsilon}{t}\right)^{2}=t^{2\alpha-2}\epsilon^{2}\leq\epsilon^{2\alpha}$. 

Thus, we have $\frac{\left(P_{0}(Y=\perp|x)-P_{1}(Y=\perp|x)\right)^{2}}{P_{0}(Y=\perp|x)+P_{1}(Y=\perp|x)}=O\left(\epsilon^{2\alpha}+\epsilon^{2}\right)$.

For $Y=1$, for the denominator, 
\begin{eqnarray*}
2\left(P_{0}(Y=1|X=t)+P_{1}(Y=1|X=t)\right) & = & t^{\alpha}\left(1-t^{\beta}\right)+\left(t-\epsilon\right)^{\alpha}\left(1-\left(t-\epsilon\right)^{\beta}\right)\\
 & \geq & t^{\alpha}\left(1-t^{\beta}\right)\\
 & \geq & t^{\alpha}\left(1-\left(\frac{3}{4}\right)^{\beta}\right)
\end{eqnarray*}

For the numerator, 
\begin{align*}
 & \left(P_{0}(Y=1|X=t)-P_{1}(Y=1|X=t)\right)^{2}\\
= & \frac{1}{4}\left(t^{\alpha}\left(1-t^{\beta}\right)-\left(t-\epsilon\right)^{\alpha}\left(1-\left(t-\epsilon\right)^{\beta}\right)\right)^{2}\\
\leq & \frac{1}{2}\left(t^{\alpha}-\left(t-\epsilon\right)^{\alpha}\right)^{2}+\frac{1}{2}\left(t^{\alpha+\beta}-\left(t-\epsilon\right)^{\alpha+\beta}\right)^{2}\\
= & \frac{1}{2}t^{2\alpha}\left(1-(1-\frac{\epsilon}{t})^{\alpha}\right)^{2}+\frac{1}{2}t^{2\alpha+2\beta}\left(1-(1-\frac{\epsilon}{t})^{\alpha+\beta}\right)^{2}\\
\leq & \frac{1}{2}t^{2\alpha}\left(1-(1-\frac{\epsilon}{t})^{\alpha}\right)^{2}+\frac{1}{2}t^{2\alpha}\left(1-(1-\frac{\epsilon}{t})^{\alpha+\beta}\right)^{2}
\end{align*}

If $\alpha\geq1$, by Lemma~\ref{lem:bernoulli}, $\frac{1}{2}t^{2\alpha}\left(1-(1-\frac{\epsilon}{t})^{\alpha}\right)^{2}+\frac{1}{2}t^{2\alpha}\left(1-(1-\frac{\epsilon}{t})^{\alpha+\beta}\right)^{2}\leq\frac{1}{2}t^{2\alpha}\left(\alpha\frac{\epsilon}{t}\right)^{2}+\frac{1}{2}t^{2\alpha}\left(\left(\alpha+\beta\right)\frac{\epsilon}{t}\right)^{2}=\left(\frac{1}{2}\alpha^{2}+\frac{1}{2}\left(\alpha+\beta\right)^{2}\right)t^{2\alpha-2}\epsilon^{2}$.
Thus, $\frac{\left(P_{0}(Y=1|x)-P_{1}(Y=1|x)\right)^{2}}{P_{0}(Y=1|x)+P_{1}(Y=1|x)}\leq\left(\frac{1}{2}\alpha^{2}+\frac{1}{2}\left(\alpha+\beta\right)^{2}\right)t^{\alpha-2}\epsilon^{2}/\left(1-\left(\frac{3}{4}\right)^{\beta}\right)$
which is $O(\epsilon^{2})$ if $\alpha\geq2$ and $O\left(\epsilon^{\alpha}\right)$
if $\alpha\leq2$.

If $\alpha\leq1$ and $\alpha+\beta\geq1$, by Lemma~\ref{lem:bernoulli},
$\frac{1}{2}t^{2\alpha}\left(1-(1-\frac{\epsilon}{t})^{\alpha}\right)^{2}+\frac{1}{2}t^{2\alpha}\left(1-(1-\frac{\epsilon}{t})^{\alpha+\beta}\right)^{2}\leq\frac{1}{2}t^{2\alpha}\left(\frac{\epsilon}{t}\right)^{2}+\frac{1}{2}t^{2\alpha}\left(\left(\alpha+\beta\right)\frac{\epsilon}{t}\right)^{2}=\left(\frac{1}{2}+\frac{1}{2}\left(\alpha+\beta\right)^{2}\right)t^{2\alpha-2}\epsilon^{2}\leq\left(\frac{1}{2}+\frac{1}{2}\left(\alpha+\beta\right)^{2}\right)t^{2\alpha-2}\epsilon^{2}$.
Thus, $\frac{\left(P_{0}(Y=1|x)-P_{1}(Y=1|x)\right)^{2}}{P_{0}(Y=1|x)+P_{1}(Y=1|x)}\leq\left(\frac{1}{2}+\frac{1}{2}\left(\alpha+\beta\right)^{2}\right)t^{\alpha-2}\epsilon^{2}/\left(1-\left(\frac{3}{4}\right)^{\beta}\right)=O\left(\epsilon^{\alpha}\right)$.

If $\alpha\leq1$, $\alpha+\beta\leq1$, by Lemma~\ref{lem:bernoulli},
$\frac{1}{2}t^{2\alpha}\left(1-(1-\frac{\epsilon}{t})^{\alpha}\right)^{2}+\frac{1}{2}t^{2\alpha}\left(1-(1-\frac{\epsilon}{t})^{\alpha+\beta}\right)^{2}\leq\frac{1}{2}t^{2\alpha}\left(\frac{\epsilon}{t}\right)^{2}+\frac{1}{2}t^{2\alpha}\left(\frac{\epsilon}{t}\right)^{2}=t^{2\alpha-2}\epsilon^{2}$.
Thus, $\frac{\left(P_{0}(Y=1|x)-P_{1}(Y=1|x)\right)^{2}}{P_{0}(Y=1|x)+P_{1}(Y=1|x)}\leq t^{\alpha-2}\epsilon^{2}/\left(1-\left(\frac{3}{4}\right)^{\beta}\right)=O\left(\epsilon^{\alpha}\right)$.

Therefore, we have $\frac{\left(P_{0}(Y=1|x)-P_{1}(Y=1|x)\right)^{2}}{P_{0}(Y=1|x)+P_{1}(Y=1|x)}=O\left(\epsilon^{\alpha}+\epsilon^{2}\right)$. 

Likewise, we can get $\frac{\left(P_{0}(Y=0|x)-P_{1}(Y=0|x)\right)^{2}}{P_{0}(Y=0|x)+P_{1}(Y=0|x)}=O\left(\epsilon^{\alpha}+\epsilon^{2}\right)$.
So we prove $\sum_{y}\frac{\left(P_{0}(Y=y|x)-P_{1}(Y=y|x)\right)^{2}}{P_{0}(Y=y|x)+P_{1}(Y=y|x)}=O\left(\epsilon^{\alpha}+\epsilon^{2}\right)$
when $x\leq\frac{1}{2}$. This concludes the proof.
\end{proof}

\subsection{Other lemmas}
\begin{lem}
(\cite{raginsky2011lower}, Lemma 4) \label{lem:packing}For sufficiently
large $d>0$, there is a subset $M\subset\left\{ 0,1\right\} ^{d}$
with following properties: (i) $\left|M\right|\geq2^{d/48}$; (ii)
$\left\Vert v-v'\right\Vert _{0}>\frac{d}{12}$ for any two distinct
$v,v'\in M$; (iii) for any $i=1,\dots,d$, $\frac{1}{24}\leq\frac{1}{M}\sum_{v\in M}v_{i}\leq\frac{3}{24}$.
\end{lem}

\begin{lem}
\label{lem:bernoulli}If $x\leq1$,$r\geq1$, then $\left(1-x\right)^{r}\geq1-rx$
and $1-\left(1-x\right)^{r}\leq rx$. 

If $0\leq x\leq1$,$0\leq r\leq1$, then $(1-x)^{r}\geq\frac{1-x}{1-x+rx}$
and $1-(1-x)^{r}\leq\frac{rx}{1-(1-r)x}\leq x$.
\end{lem}
Inequalities above are know as Bernoulli's inequalities. One proof
can be found in \cite{li2013some}.
\begin{lem}
\label{lem:diff in E}Suppose $\epsilon,\tau$ are positive numbers
and $\delta\leq\frac{1}{2}$. Suppose $\left\{ Z_{i}\right\} _{i=1}^{\infty}$
is a sequence of i.i.d random variables bounded by 1, $\mathbb{E}Z_{i}\geq\tau\epsilon$,
and $\text{Var}(Z_{i})=\sigma^{2}\leq2\epsilon$. Define $V_{n}=\frac{n}{n-1}\left(\sum_{i=1}^{n}Z_{i}-\frac{1}{n}\left(\sum_{i=1}^{n}Z_{i}\right)^{2}\right)$,
$q_{n}=q\left(n,V_{n},\delta\right)$ as Procedure~\ref{alg:test}.
If $n\geq\frac{\eta}{\tau\epsilon}\ln\frac{1}{\delta}$ for some sufficiently
large number $\eta$ (to be specified in the proof), then with probability
at least $1-\delta$ , $\frac{q_{n}}{n}-\mathbb{E}Z_{i}\leq-\tau\epsilon/2$.
\end{lem}
\begin{proof}
By Lemma~\ref{lem:single-var-bound}, with probability at least $1-\delta$,
$V_{n}\leq4n\sigma^{2}+8\ln\frac{1}{\delta}$, which implies 
\[
q_{n}\leq D_{1}\left(1+\ln\frac{1}{\delta}+\sqrt{\left(4n\sigma^{2}+9\ln\frac{1}{\delta}+1\right)\left(\left[\ln\ln\right]_{+}(4n\sigma^{2}+9\ln\frac{1}{\delta}+1)+\ln\frac{1}{\delta}\right)}\right)
\]
 We denote the RHS by $q$. 

On this event, we have

\begin{eqnarray*}
\frac{q_{n}}{n}-\mathbb{E}Z_{i} & \leq & \frac{q}{n}-\tau\epsilon\\
 & = & \tau\epsilon\left(\frac{q}{n\tau\epsilon}-1\right)\\
 & \stackrel{(a)}{\leq} & \tau\epsilon\left(\frac{2D_{1}}{\eta}+\frac{D_{1}}{\eta\ln\frac{1}{\delta}}\sqrt{\frac{9\eta}{\tau}\ln\frac{1}{\delta}\left(\left[\ln\ln\right]_{+}(\frac{9\eta}{\tau}\ln\frac{1}{\delta})+\ln\frac{1}{\delta}\right)}-1\right)\\
 & = & \tau\epsilon\left(\frac{2D_{1}}{\eta}+D_{1}\sqrt{\frac{9}{\eta\tau\ln\frac{1}{\delta}}\left[\ln\ln\right]_{+}(\frac{9\eta}{\tau}\ln\frac{1}{\delta})+\frac{9}{\eta\tau}}-1\right)
\end{eqnarray*}

where (a) follows from $\frac{q}{n}$ being monotonically decreasing
with respect to $n$. By choosing $\eta$ sufficiently large, we have
$\frac{2D_{1}}{\eta}+D_{1}\sqrt{\frac{9}{\eta\tau\ln\frac{1}{\delta}}\left[\ln\ln\right]_{+}(\frac{9\eta}{\tau}\ln\frac{1}{\delta})+\frac{9}{\eta\tau}}-1\leq-\frac{1}{2}$,
and thus $\frac{q_{n}}{n}-\mathbb{E}Z_{i}\leq-\tau\epsilon/2$.
\end{proof}

\end{document}